\documentclass[10pt]{llncs}
\usepackage{mathtools} 
\usepackage{stmaryrd} 
\usepackage{enumitem}  
\setlist[description]{style=nextline} 
\usepackage{xcolor}
\usepackage{adjustbox} 
\usepackage{url} 
\usepackage{todonotes}
\usepackage{tikz}  
\usepackage{tikz-cd}  
\DeclarePairedDelimiterX{\infdivx}[2]{(}{)}{%
  #1\;\delimsize\|\;#2%
}

\usetikzlibrary{matrix,fit,decorations.pathmorphing,intersections,arrows,positioning,shapes.misc}  
\usepackage{amssymb}

\newcommand{\range}{\mathsf{range}}

\newcommand{\evalit}{\mathsf{eval}}

\newcommand{\hide}[1]{}

\begin{document}
     \title{Discussion Graph Semantics of
      First-Order\\ Logic with Equality for Reasoning 
      about \\ Discussion and Argumentation} 

\author{{\ }\quad\qquad\qquad\qquad Ryuta Arisaka\orcidID{0000-0003-3203-1517}}   
\institute{Department of Informatics, Kyoto University, Japan \\
\email{ryutaarisaka@gmail.com} 
} 

    \maketitle
    
\begin{abstract}          
	We make three contributions. First, we formulate a {\it discussion-graph semantics for 
	first-order logic with equality}, enabling reasoning about discussion and argumentation 
	in AI more generally than before. 
	This addresses the current lack of a formal reasoning framework capable of handling diverse discussion and argumentation models. 
	Second, we {\it generalise Dung’s notion of extensions to cases where two or more graph nodes in an argumentation framework are equivalent}.
	Third, we connect these two contributions by showing that the {\it generalised extensions are first-order characterisable} within the proposed discussion-graph semantics.  
        {\it Propositional characterisability of all Dung's extensions} is an immediate consequence. 
	We furthermore show that the {\it set of all generalised extensions (acceptability semantics), too, are first-order characterisable}. 
	{\it Propositional characterisability of all Dung's acceptability semantics} is an immediate consequence. 
  \end{abstract} 

\section{Introduction}  \label{sec_introduction}   

Discussion and argumentation are fundamental to multi-agent communication, and numerous 
models have been proposed to represent them. Some—such as Toulmin’s model \cite{Toulmin58}, issue-based information systems \cite{Kunz70}, and argumentation schemes \cite{Walton08}—serve as guidelines for structuring effective discussions. Others, following Dung’s tradition \cite{Dung95,Dung09,Besnard01}, provide logical reasoning capabilities grounded in formal principles. Most of these models are graph-based (or reducible to graphs).
Each has found application in diverse fields, including AI and Law \cite{Bench-Capon03} 
and 
Education \cite{Hsu15}. Yet, with respect to formal reasoning, the Dung-style models have received the most attention.

Since the emergence of ChatGPT in 2022, the ability to automatically extract graphical representations of everyday discussions and argumentations has grown rapidly. This motivates a 
broader framework for reasoning about such structures. Even seemingly simpler reasoning tasks—such as detecting whether a discussion exhibits the structure of Toulmin’s model, 
an issue-based system, or an argumentation scheme—can be practically valuable. 
When a pattern is absent, a human or large language model could prompt participants to provide missing elements.

Despite this demand, there is a conspicuous lack of a formal reasoning framework  
for general discussion graphs. Prior proposals \cite{Villata12,Doutre14} remain specific to 
Dung’s argumentation model, are constructed bottom-up, and are largely propositional. 
How they might extend to broader discussion structures is not obvious. 
Seeing the gap, after preliminaries, we make three contributions. 

 \textbf{First (Section 3).}
We {\it develop a reasoning framework for general discussion and argumentation} in a top-down manner. To ensure compatibility with existing formal logic, we do not introduce a new logic but adopt the syntax of first-order logic with equality. What we newly formulate is its \emph{discussion-graph semantics}, enabling direct reasoning over discussion and argumentation graph structures. 
Technically, we define the semantics of first-order logic formulas 
where the domain of discourse---the object-level semantic structure---is a discussion graph with annotations on nodes and edges. The key research question is as follows: {\it in first-order 
logic over sentences, a sentence (of arbitrary finite length) is the basic building 
block; analogously, in first-order logic over discussion graphs, a discussion graph (of arbitrary finite size) should serve as the basic building block. How can this intuition be captured in the semantics of first-order logic (with equality)?} We answer this by allowing predicate symbols to denote annotated graph structures.

 \textbf{Second (Section 4).} 
We show that such annotated graph structures help extend 
Dung's argumentation model naturally. 
We present an {\it equivalence-equipped Dung's model} 
and then define a corresponding set of {\it novel  ` extensions'}. 

\textbf{Third (Section 5).}
We connect these two contributions by showing that {\it all the generalised extensions are 
first-order characterisable} within the proposed discussion-graph semantics.  
The closest prior contribution to ours is in \cite{Doutre14} in which propositional 
characterisability of some types of Dung’s ‘extensions’ were shown. 
That of other Dung’s ‘extensions’ were left open, however. En route, we close the 
gap as a technical corollary of our contribution. Furthermore, we prove that 
the set of all extensions of a given type, too, can be similarly characterised. 
This fully answers the question as to how expressive a logic must be to characterise Dung's argumentation---propositional logic is sufficient. 

\vspace{-0.3cm}

\section{Preliminaries: the Syntax of First-Order Logic}  \label{sec_technical_preliminaries} 
We introduce the syntax of first-order logic 
with equality succinctly. 

A {\it logical connective} is a member of the set $\{\top, \bot, \forall, 
\exists, \neg, \wedge, \vee, \supset\}$. $\top$ and $\bot$ 
have arity 0; $\forall$, $\exists$ and $\neg$ have arity 1; and   
$\wedge$, $\vee$ and $\supset$ have arity 2. 
	A {\it variable} is a member of an uncountable set $\pmb{Vars}$.  
	A {\it logical symbol} is one  of the following: a logical connective, a variable, 
	a parenthesis/bracket symbol ({\it e.g.} `(', `[',  `)', `]'), 
	a punctuation symbol ({\it e.g.} `.', `,') 
	or an equality symbol =. 
	
	A {\it function symbol} is a member of an uncountable 
	set $\pmb{Funcs}$  
	and each function symbol has an arity of some non-negative 
	integer. It is assumed that there are infinitely many 
	function symbols for each arity in $\pmb{Funcs}$. 
	A {\it constant} is a function symbol with arity 0. 
	A {\it predicate symbol} is a member of an uncountable set $\pmb{Preds}$ 
	and each predicate symbol has an arity of some non-negative 
	integer. It is again assumed that there are infinitely many 
	predicate symbols for each arity in $\pmb{Preds}$. 
	A {\it propositional variable} is a predicate symbol 
	with arity 0. A {\it non-logical symbol} 
	is one of the following: a function symbol (including a constant) or a  
	predicate symbol (including a propositional variable). 
	  
The language for first-order logic, \textsf{FOL}, comprises  
all the logical and non-logical symbols. 
For convenience, any variable is denoted by $x, y, z$, 
 any predicate symbol by $p$ and any function symbol by $f$. 
 By $t$, we shall denote a {\it term} and 
 by $\overrightarrow{t}$ we shall denote a tuple of 
 $n \geq 0$ terms $(t_1, \ldots, t_n)$. The precise 
 definition of a term is as follows. (1) Any variable is a term. 
 (2) Any $f(\overrightarrow{t})$ for some function symbol $f$  
 of arity $n \geq 0$ and some tuple of $n$ terms $\overrightarrow{t}$ 
 is a term. (3) Any term is recognised by (1) and (2) alone. 

A {\it formula} $F$ is any of the following. 
(1) $p(\overrightarrow{t})$ for some predicate symbol $p$ of arity $n \geq 0$ and  
	some tuple of $n$ terms $\overrightarrow{t}$.  
	(2) $\top$ or $\bot$. (3) 
		$\neg F_1$ for some formula $F_1$.  
		(4) $\forall x.F_1$ or $\exists x.F_1$  
			for some variable $x$ and some formula $F_1$.  
		(5) $F_1 \wedge F_2$, $F_1 \vee F_2$ or 
		      $F_1 \supset F_2$ for some formulas $F_1$ and $F_2$.  
	      (6) $(F_1)$ or $[F_1]$ for some formula $F_1$.  
	      (7) $t_1 = t_2$ for some terms $t_1$ and $t_2$. 
A formula $F$ is {\it well-formed} iff there is no variable 
occurring free 
in $F$.  As for the binding
order of the logical connectives, $\neg$ binds strongest, 
$\wedge$ and $\vee$ bind the second strongest, 
$\forall$ and $\exists$ the third strongest, while $\supset$ binds the weakest. 

	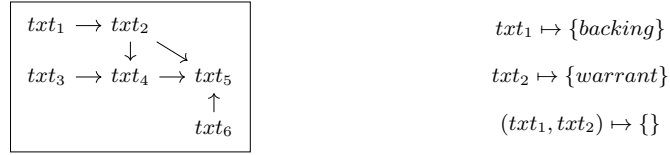
\begin{figure}[!t] 
	\begin{center} 
		\adjustbox{scale=0.9}{ 
	      \begin{tikzcd}[column sep=small,row sep=small,/tikz/execute at end picture={
			      \node (large) [rectangle, draw, fit=(A1) (A2) (A3),
			      label=above:Toulmin's model \cite{Toulmin58} 
			      formatted into an annotated graph] {}; 
  }] 
		      |[alias=A1]| txt_1:\{backing\} \arrow[r]{}{\{\}} & txt_2:\{warrant\} 
		      \arrow[d]{}{\{\}} \arrow[dr]{}{\{\}}\\ 
		      |[alias=A3]|	txt_3:\{grounds\} \arrow[r]{}{\{\}} & 
		      txt_4:\{qualifier\} \arrow[r]{}{\{\}} 
		      & txt_5:\{claim\}\\ 
		       & & |[alias=A2]| txt_6:\{rebuttal\} \arrow[u,swap]{}{\{\}}\\\\
	\end{tikzcd} 
		}  
	    
		\adjustbox{scale=0.9}{
		\begin{tikzcd}[column sep=small,row sep=small,/tikz/execute at end picture={
			      \node (large) [rectangle, draw, fit=(A1) (A2) (A3),
			      label=above:The graph 
			      part of the above annotated 
			      graph] {}; 
  }] 
		       |[alias=A1]| txt_1 \arrow[r]{}{} & txt_2
			\arrow[d]{}{} \arrow[dr]{}{}\\ 
		      |[alias=A3]|	txt_3 \arrow[r]{}{} & 
		      txt_4 \arrow[r]{}{}  
		      & txt_5\\ 
		       & & |[alias=A2]| txt_6 \arrow[u,swap]{}{}
	\end{tikzcd} 
	
	\begin{tikzcd}[row sep=tiny,/tikz/execute at end picture={\node 
		(large) [rectangle, fit=(A1) (A2) (A3),
		label=above:{\small Assignment of annotations partially shown}]{};}] 
		txt_1 \mapsto \{backing\}\\ 
		txt_2 \mapsto \{warrant\} \\ 
		(txt_1, txt_2) \mapsto \{\} 
	\end{tikzcd} 
		}
	\end{center}   
		\caption{\pmb{Top}: An example
		of a Toulmin's model as an `annotated 
		graph' comprising a graph and annotations on nodes and edges. 
		\pmb{Bottom left}: The graph part of the annotated graph. 
		\pmb{Bottom right}: Assignment 
		of annotations, partially shown 
		for two nodes $txt_1$ and $txt_2$ and 
		one edge $(txt_1, txt_2)$.  
		$backing$ is assigned to $txt_1$, 
		$warrant$ is assigned to $txt_2$, 
		and no annotation is assigned to $(txt_1, txt_2)$. 
		}  
		\label{fig_toulmin} 
	\end{figure} 

\section{Discussion Graph Semantics of First-Order Logic} \label{sec_discussion_graph_semantics_first_order_logic_equality} 
When discussion models are 
extracted from natural language texts 
through mining \cite{Habernal17,Lawrence20}, 
they typically appear as graphs with annotations on nodes and 
edges. It is therefore natural to take such structures 
as the domain of discourse
in our semantics.  We call a pair consisting of 
(i) a graph and (ii) a function 
assigning annotations (as strings) to its nodes and edges 
an {\it annotated graph}.    
The upper part of 
Fig. \ref{fig_toulmin} shows one example of annotated graph representing Toulmins' model. 
This is a compact form of the underlying graph, shown at the lower left, 
and a function to assign annotations, partially shown at the lower right.  
In the remainder, annotated graphs are presented in this compact form. 

To explain the annotated graph in Fig. \ref{fig_toulmin} a little more in detail,  
$txt_i$ denotes some natural language text (as a string) 
and the empty annotation $\{\}$ (as an empty string) 
signifies a `{\it whatever}' or `{\it don't care}'
annotation. 
Now, Toulmin's model is not originally a graph: an edge from $txt_2$ 
	should actually go to the edge from $txt_4$ to $txt_5$.    
	We have converted the network into the annotated graph  
	with one edge from $txt_2$ to $txt_4$  
	and another edge from $txt_2$ to $txt_5$. 
      
	\begin{figure}[!t]
	\begin{center}  
		\adjustbox{scale=0.9}{
	      \begin{tikzcd}[row sep=small,/tikz/execute at end picture={
			      \node (large) [rectangle, draw, fit=(A0) (A1) (A2) (A3) (A10),
			      ] {}; 
  }]    
		      & & |[alias=A0]|  \\
		       |[alias=A1]| txt_1:\{\} \arrow[r,bend right=15,swap]{}{\{attacks\}}
		      & txt_2:\{\} \arrow[l, shift left, bend right=15,swap]{}{\{attacks\}} 
		      \arrow[r]{}{\{attacks\}} 
		      & txt_3:\{\} \arrow[r]{}{\{attacks\}} & txt_4:\{\}
		      \arrow[r]{}{\{attacks\}} & 
		      |[alias=A2]| 
		      txt_5:\{\} \arrow[ll, bend right=20,swap]{}{\{attacks\}} \\
		      & & |[alias=A10]| 
	\end{tikzcd} 
} 
	\end{center}   
		\caption{An example of Dung's model as an annotated graph.} 
		\label{fig_dung_kunz} 
	\end{figure}
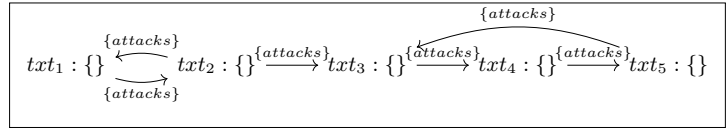 

      For other discussion models, Fig. \ref{fig_dung_kunz} shows one 
      example of Dung's model. 
      The original Dung's model does not care about 
      annotations on texts/statements, so every text/statement is given $\{\}$ 
      as its annotation. On the other hand, every edge 
      is given $attacks$ annotation. 
	
\begin{figure}[!h] 
	\begin{center} 
		\adjustbox{scale=0.9}{
\begin{tikzcd}[/tikz/execute at end picture={
			      \node (large) [rectangle, draw, fit=(A0) (A1) (A2),
			      ] {}; 
		      }]  
		      |[alias=A0]| {\rm \small \text{`That's too much of a complaint'}}:\{complaint\}
		       \arrow[r]{}{\{attacks\}}
		      & {\rm \small \text{`The sea is too far'}}:\{issue\} 
		      \arrow[d]{}{\{attacks\}} \\ 
		       {\small \text{`Let's go swimming'}}:\{suggestion\}
		       & |[alias=A1]|  {\small \text{`We go fishing'}}:\{opinion\} \arrow[l]{}{\{attacks\}}\\
		       |[alias=A2]| {\small \text{`Is it warm, though?'}:\{concern\}} 
		       \arrow[u]{}{\{questions\}}
	\end{tikzcd} 
}
	\end{center}    
	\caption{Another example of annotated graph.} 
	\label{fig_graph_domain_of_discourse} 
\end{figure}
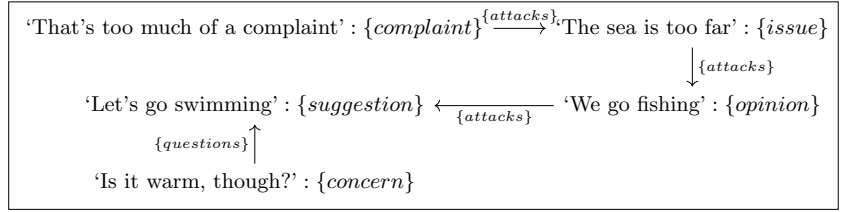 

In practice, we may not obtain an annotated graph that perfectly 
matches one of these well-known discussion models 
structure-wise. 
Fig. \ref{fig_graph_domain_of_discourse} 
shows one example which is close to Dung's model  
but with $questions$ annotation on one of the edges 
and with various annotations on texts/statements. 

Consequently, if we are to reason generally about annotated graphs, 
each of them should be recognisable. We thus put forward the following definition.  
Here and elsewhere, 
$\mathfrak{p}(\ldots)$ denotes the power set of $\ldots$

\begin{definition}[Object-level annotated graphs] \rm    
	An {\it object-level statement} is a string.  
	An {\it annotation} is a string. 
	An {\it object-level annotated graph} is a pair of: a graph   
	$(ObjStmts, ObjE)$ where $ObjStmts$ is 
	a set of object-level statements and $ObjE$ is a binary 
	relation over them; 
	and a function $Obj\Pi: (ObjStmts \cup ObjE)  
	\rightarrow \mathfrak{p}(Annos)$  
	where $Annos$ is a set of annotations. 
	\hfill$\spadesuit$ 
\end{definition}   
To be able to tell 
whether one object-level annotated graph is smaller than another, 
we define the following order  $\unlhd$. 

\begin{definition}[Order on object-level annotated 
	graphs] \rm  $\unlhd$ is a binary relation over object-level 
	annotated graphs defined 
	as follows. 
	$((ObjStmts_1, ObjE_1),\linebreak Obj\Pi_1) \unlhd 
	((ObjStmts_2, ObjE_2), Obj\Pi_2)$ iff  
	the following conditions hold. 
	\begin{itemize} 
		\item $(ObjStmts_1, ObjE_1)$ is a subgraph 
			of $(ObjStmts_2, ObjE_2)$, 
			{\it i.e.} $ObjStmts_1 \subseteq ObjStmts_2$ 
			and $ObjE_1 \subseteq ObjE_2$.\footnote{Or 
			$ObjE_1 \subseteq (ObjE_2 \cap (ObjStmts_1 \times 
			ObjStmts_1))$. But this is redundant 
			since $(ObjStmts_1, ObjE_1)$ must be a graph.} 
		\item For every member $u$ of $ObjStmt_1 \cup ObjE_1$,  
			$Obj\Pi_1(u) \subseteq Obj\Pi_2(u)$. \hfill$\spadesuit$ 
	\end{itemize} 
\end{definition} 
\begin{proposition} 
     $\unlhd$ is a preorder. 
\end{proposition}  
\begin{proof} 
    For any $\eta_1$, $\eta_2$ and $\eta_3$ each denoting an object-level 
	annotated graph, 
     $\eta_1 \unlhd \eta_1$ trivially; and $\eta_1 \unlhd \eta_2$ and $\eta_2 
	\unlhd \eta_3$ 
	materially imply $\eta_1 \unlhd \eta_3$ trivially.  \hfill$\Box$ 
\end{proof}

\subsection{The role of predicate symbols and variables} 
To clarify the role of predicate symbols and variables in our semantics, 
we first recall how statements are treated in first-order logic.
Consider two sentences 
``{\it Tom is a nice person.}'' 
and ``{\it Tom gives me flowers.}''  
Let these be represented  
via precicate symbols $p_1$ and $p_2$. 

Suppose we are interested in representing   
the sentence ``{\it If Tom is a nice person, 
Tom gives me flowers.}" with them. 
When they have arity 0, 
$p_1 \supset p_2$ represents the new sentence.  
The internal content of $p_1$ and $p_2$ are inaccessible, so they have to be handled atomically. 
By allowing 
a positive arity, we move from concrete sentences to 
schematic sentences. For example, 
if $p_1$ has arity 1 and $p_1(Tom)$ 
denotes the corresponding sentence, $p_1$ effectively represents  
the following skeleton sentence ``{\it $\star 1$ is a nice person}.'' 
which has a placeholder $\star 1$ to be replaced by $Tom$. 
Similarly, if $p_2$ has arity 3 and 
$p_2(Tom, me, flowers)$ denotes the corresponding sentence, 
$p_2$ effectively represents  
the following skeleton sentence ``{\it $\star 1$ gives $\star 2$ $\star 3$},'' 
\begin{figure}[!t] 
	\begin{center} 
	      \begin{tikzcd} 
		      u_1:\{\alpha_1\}
		       \arrow[r]{}{\{\alpha_2\}}
		      & u_2:\{\alpha_1\}
		      \arrow[d]{}{\{\alpha_2\}} \\ 
		       u_3:\{\alpha_3, \alpha_2\}
		       &  u_4:\{\alpha_4\} \arrow[l]{}{\{\alpha_2\}}\\
		       u_5:\{\alpha_5\}
		       \arrow[u]{}{\{\alpha_6\}}
	\end{tikzcd} 
		{\ }$\qquad$ $\qquad$$\qquad$$\qquad$ 
		\begin{tikzcd} 
			{\color{purple}{\star 1}}:\{\alpha_1\}
			\arrow[r]{}{\{\alpha_2\}}
			& u_2:{\color{purple}{\{\star 2\}}}
			\arrow[d]{}{\{\alpha_2\}} \\ 
			u_3:\{\alpha_3, {\color{purple}{\star 3}}\}
			&  u_4:\{\alpha_4\} 
			\arrow[l]{}{{\color{purple}{\{\star 3\}}}}\\
			u_5:\{\alpha_5\}
			\arrow[u]{}{\{\alpha_6\}}
	\end{tikzcd} 

	\end{center}    
	\caption{\pmb{Left}: An 
	example of object-level annotated graph.  
	$u_i$ ($1 \leq i \leq 5$) is a member of $ObjStmts$ 
	and $\alpha_j$ ($1 \leq j \leq 6$) is a member of 
	$Annos$. 
	\pmb{Right}: An example of skeleton 
	annotated graph with 3 placeholders.} 
	\label{fig_predicate_graph} 
\end{figure}
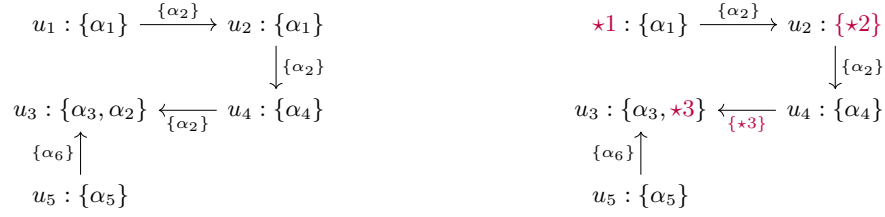 
with each placeholder replaced by the corresponding constant in the tuple. 
Instead of a constant, we can of course pass a variable to a placeholder 
and quantify it. 

{\it We take an analogy} for our semantics.  
Just as (complete) sentences were building blocks, 
we treat (complete) annotated graphs 
as building blocks.  
Let the annotated graph in Fig. \ref{fig_graph_domain_of_discourse} be such a building block. 
We denote it—or a schematic version of it—by a predicate symbol 
$p$.
When $p$'s arity is 0, $p$ denotes the object-level annotated graph itself; when the arity is 
a positive integer $n$, it denotes a skeleton annotated graph with 
$n$ placeholders.

Fig. \ref{fig_predicate_graph} illustrates this. 
On the left is an object-level annotated graph where 
$u_i$ ($1 \leq i \leq 5$) is an object-level statement 
and $\alpha_i$ ($1 \leq i \leq 6$) is an annotation. 
When $p$'s arity is 0, it denotes this entire object-level 
annotated graph. When the arity is 3,  
$p$ could denote the skeleton annotated graph on the right.  
By substituting 
$u_1$ for the first placeholder, 
$\alpha_1$ for the second 
and $\alpha_2$ for both occurrences of the third placeholder, we obtain $p(u_1, \alpha_1, \alpha_2)$ denoting the 
same object-level annotated graph on the left. Here, we may say that the skeleton annotated 
graph as denoted by $p$ is {\it instantiable} by 
the tuple of members of object-level statements and/or annotations  
$(u_1, \alpha_1, \alpha_2)$. 

With this intuition, the definition of a 
skeleton annotated graph is as follows. 

\begin{definition}[Skeleton annotated graphs]  \rm  
	A {\it special natural number} is  
	a natural number prefixed by $\star$.\footnote{$\star1$, $\star2$, 
	and so on are all special natural numbers.} 
	A {\it skeleton statement} is either an object-level 
	statement or a special natural  number. A {\it skeleton annotation} 
	is either an annotation or a special natural number. 
	A {\it skeleton annotated graph} is a pair of: 
	a graph $(SkelStmts,SkelE)$ where $SkelStmts$ is  
	a set of object-level statements and/or 
	special natural numbers and 
	$SkelE$ is a binary 
	relation over them; and a function $Skel\Pi: (SkelStmts \cup 
	SkelE) \rightarrow \mathfrak{p}(SkelAnnos)$ 
	where $SkelAnnos$ is a set of annotations and/or special natural numbers. 

	A  {\it degree-$n$ skeleton annotated graph} 
	is a skeleton annotated graph in which 
	$\star1, \ldots, \star n$ all occur 
	but not any $\star j$ for $n < j$.  
	No special natural numbers occur in a degree-0 skeleton 
	annotated graph. 
	\hfill$\spadesuit$ 
\end{definition} 

\noindent The definition of instantiability of a degree-$n$ 
skeleton annotated graph is: 
\begin{definition}[Instantiability] \label{ex_instantiability}\rm 
	Given a degree-$n$ skeleton annotated graph\linebreak 
	$((SkelStmts, SkelE), Skel\Pi)$ 
	and 
	a member $(u_1, \ldots, u_n)$ 
	of 
	{\small $(ObjStmts \cup Annos)^{\times n}$}, \linebreak
	let $((SkelStmts, SkelE), Skel\Pi)[u_1 \mapsto \star 1, \ldots, u_n \mapsto \star n]$ 
	denote an object-level annotated graph as the result of 
	simultaneously substituting $u_i$ into $\star i$ ($1 \leq i \leq n$) 
	occurring in $SkelStmts \cup SkelAnnos$. 
	 
	Then, given an object-level annotated graph $(ObjG, Obj\Pi)$, 
	a degree-$n$  skeleton annotated graph $(SkelG, Skel\Pi)$,   
	and 
	a member $(u_1, \ldots, u_n)$ 
	of 
	$(ObjStmts \cup Annos)^{\times n}$, we say 
	{\it $(u_1, \ldots, u_n)$ instantiates $(SkelG, Skel\Pi)$   
	below $(ObjG, Obj\Pi)$} 
			  iff 
			  $(SkelG, Skel\Pi)[u_1 \mapsto \star 1,\ldots, 
			   u_n \mapsto \star n] \unlhd (ObjG, Obj\Pi)$.   
			   \hfill$\spadesuit$ 
\end{definition} 
\begin{figure}[!t] 
	\begin{tikzcd} 
		{\star 1}:\{u_1\}
		       \arrow[r,bend right=15,swap]{}{\{\star 2\}}
			& {u_2}:{\{u_3\}}
		      \arrow[l, shift left, bend right=15,swap]{}{\{\star 2\}} 
	\end{tikzcd}    
		 {\ }\qquad\qquad\qquad
		 \begin{tikzcd} 
		{u_4}:\{u_1\}
		       \arrow[r,bend right=15,swap]{}{\{u_5, u_6\}}
			& {u_2}:{\{u_3\}}
		      \arrow[l, shift left, bend right=15,swap]{}{\{u_5\}} 
			 \arrow[r,bend right=15,swap]{}{\{u_7\}}
			& {u_8}:{\{u_1\}}
		      \arrow[l, shift left, bend right=15,swap]{}{\{u_7\}} 
	\end{tikzcd}  

		 \caption{\pmb{Left}: a degree-2 skeleton annotated graph.
		 \pmb{Right}: an object-level annotated graph. 
		 }
		 \label{fig_small_fig} 
	 \end{figure}
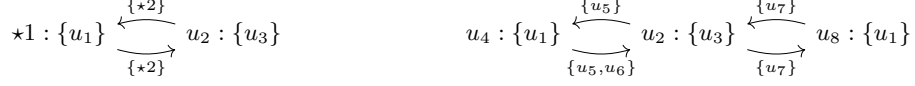 

\begin{example}[Instantiability] 
	 Let $u_1$ and $u_3$ be a member of $Annos$ 
	 and let $u_2$ be a member of $ObjStmts$, 
	 then the left graph in Fig. \ref{fig_small_fig} 
	 is a \mbox{degree-2} skeleton annotated graph.  
	 Meanwhile, an object-level annotated graph is on the right hand 
	 side of Fig. \ref{fig_small_fig}. 
	 	 It holds that  
	 both $(u_4, u_5)$ and $(u_8, u_7)$ instantiate 
	 the \mbox{degree-2} skeleton annotated graph  
	 below the object-level annotated graph. 
	 \hfill$\clubsuit$ 
\end{example} 

\subsection{Symbolic representation of annotated graphs}  
So far, we have focused on object-level statements and annotations. 
To move to the syntactic level, where reasoning is conducted over terms (see Section 
\ref{sec_technical_preliminaries}), 
we next define the symbolic form of an annotated graph.
\begin{definition}[Typed discussion graphs] \rm   
	A {\it typed discussion graph}  
	is a pair of: a graph $(V, E)$; and a typing function 
	$\mathcal{T}$, where each node is a term and  
	$\mathcal{T}$ assigns a set of terms to each member of $V \cup E$.  
	A {\it skeleton typed discussion graph}  
	is a pair of: a graph $(SkelV, SkelE)$; and 
	a typing function $Skel\mathcal{T}$, where each node is a term 
	or a special natural number and $Skel\mathcal{T}$ assigns 
	a set of terms and/or special natural numbers to each 
	member of $SkelV \cup SkelE$.    

	A {\it degree-$n$ skeleton typed discussion graph}  
	is a skeleton typed discussion graph in which 
	$\star 1, \ldots, \star n$ all occur but not any 
	$\star j$ for $n < j$. No special natural numbers 
	occur in a degree-0 skeleton typed discussion graph. 
	\hfill$\spadesuit$ 
\end{definition} 
These symbolic representations allow 
for graphical description of predicate symbols. 
Instead of some letter(s) $p$,  
we can agree to use the skeleton typed discussion graph, {\it e.g.} 
\adjustbox{scale=0.7}{[\begin{tikzcd} 
		{\star 1}:\{c_1\}
		       \arrow[r,bend right=15,swap]{}{\{\star 2\}}
			& {c_2}:{\{c_3\}}
		      \arrow[l, shift left, bend right=15,swap]{}{\{\star 2\}} 
\end{tikzcd}],    
}
as a more descriptive version of $p$. 
\subsection{Discussion graph semantics}  
The preceding subsections touched upon 
the domain of discourse and interpretation 
of mostly predicate symbols. Our semantic structure is formally: 
\begin{definition}[Discussion graph structures]  \rm 
	The {\it domain of discourse} is some 
	object-level annotated graph $((ObjStmts, ObjE), Obj\Pi)$. 
	An {\it interpretation} $\mathcal{I}$ is a function defined as follows. 
	(1) For every function symbol $f$ of arity $n \geq 0$, 
	   $\mathcal{I}(f)$ is a function $(ObjStmts \cup Annos)^{\times n} 
	\rightarrow (ObjStmts \cup Annos)$. In case $n = 0$, 
	 $\mathcal{I}(f)$ is a member of $(ObjStmts \cup Annos)$. 
	(2) For every predicate symbol $p$ of arity $n \geq 0$, 
	  $\mathcal{I}(p)$ is a degree-$n$ skeleton annotated 
	  graph $((SkelStmts, SkelE), Skel\Pi)$ where 
	  $SkelStmts \cup \range(Skel\Pi)$ is a subset of  
	  $ObjStmts \cup Annos \cup \{\star 1, \ldots, \star n\}$. 
	  A {\it variable assignment} $\mu$ is a function 
	$\pmb{Vars} \rightarrow (ObjStmts \cup Annos)$.   
	
	An {\it evaluation} $\evalit$ is a pair of some interpretation $\mathcal{I}$ 
	and some variable assignment $\mu$. The following rules are enforced. 
	\begin{itemize} 
		\item For any 
			variable $x \in \pmb{Vars}$,  $\evalit(x)$ 
			denotes $\mu(x)$. 
		\item For any function symbol $f \in \pmb{Funcs}$,  
			$\evalit(f)$ denotes $\mathcal{I}(f)$. 
		\item For any function symbol $f$ of arity $n$ 
			and an $n$-tuple of terms $(t_1, \ldots, t_n)$, \linebreak
			$\evalit(f(t_1, \ldots, t_n))$ 
			denotes $\evalit(f)(\evalit(t_1), \ldots, \evalit(t_n))$. 
		\item For any predicate symbol $p \in \pmb{Preds}$, 
			$\evalit(p)$ denotes $\mathcal{I}(p)$.  
			\begin{itemize} 
				\item for a graphical 
					predicate {\small [$((SkelV, SkelE), Skel\mathcal{T})$]}, 
					{\small $\evalit([((SkelV, SkelE),\linebreak Skel\mathcal{T})])$} 
					is a skeleton annotated graph 
					as the result of 
					replacing every term $t$ 
					occurring in 
					{\small $((SkelV, SkelE), Skel\mathcal{T})$} 
					by $\evalit(t)$. 
			\end{itemize} 
	\end{itemize}  
	A {\it discussion graph structure} is 
	a tuple {\small $((ObjStmts, ObjE), Obj\Pi, \evalit)$}   
	of the domain of discourse $((ObjStmts, ObjE), Obj\Pi)$
	and an evaluation $\evalit$.  \hfill$\spadesuit$ 
\end{definition}

The central task in defining 
the semantics of an existing logic is to specify how atomic 
formulas are evaluated. 
In our case, a given discussion graph structure 
evaluates $p(\overrightarrow{t})$ true  
iff the skeleton annotated graph $\evalit(p)$  
is instantiable by $\evalit(\overrightarrow{t})$ 
in such a way that the instantiated annotated graph 
is smaller than the object-level annotated graph.   
\begin{definition}[Satisfaction] \rm 
	For any discussion graph structure 
	$(\mathcal{M}, \evalit)$ with $\mathcal{M} \equiv 
	((ObjStmts, ObjE), Obj\Pi)$, 
	and any formulas, 
	we define the satisfaction relation $\models$ as follows.  
	{\small 
   \begin{itemize} 
	   \item $\mathcal{M}, \evalit \models t_1 = t_2$ iff
		   $\evalit(t_1) = \evalit(t_2)$ in 
		   $ObjStmts \cup Annos$. 
	   \item $\mathcal{M},\evalit \models p(\overrightarrow{t})$   
		   iff $\evalit(\overrightarrow{t})$ instantiates 
		   $\evalit(p)$ below $\mathcal{M}$. 
	   \item $\mathcal{M},\evalit \models \top$. 
	   \item $\mathcal{M},\evalit \not\models \bot$ (It is not the case 
		   that $\mathcal{M}, \evalit \models \bot$.) 
	   \item $\mathcal{M},\evalit \models \neg F$ iff $\mathcal{M},\evalit
		   \not\models F$. 
	   \item $\mathcal{M},\evalit \models F_1 \wedge F_2$ iff, 
		   for each $i \in \{1,2\}$, 
		   $\mathcal{M},\evalit \models F_i$. 
	   \item $\mathcal{M},\evalit \models F_1 \vee F_2$ iff, 
		   for at least one of $i \in \{1,2\}$, 
		   $\mathcal{M},\evalit \models F_i$. 
	   \item $\mathcal{M},\evalit \models F_1 \supset F_2$ 
		   iff $\mathcal{M},\evalit \not\models F_1$ 
		   or $\mathcal{M},\evalit \models F_2$.   
	   \item $\mathcal{M},\evalit \models \forall x.F$ 
		   iff, for every $\evalit'$, $\mathcal{M}, \evalit' \models F$ where 
		   $\evalit'$ is almost exactly $\evalit$ 
		   except that the variable assignment in $\evalit'$ 
		   may differ from $\evalit$'s variable assignment for 
		   the variable $x$.  
	   \item  $\mathcal{M}, \evalit \models \exists x.F$ 
		   iff there is some $\evalit'$ such that $\mathcal{M}, \evalit' \models F$ 
		   where $\evalit'$ is almost exactly $\evalit$ 
		   except that the variable assignment in $\evalit'$ 
		   may differ from $\evalit$'s variable assignment for the variable $x$.   
	      \end{itemize}  
	We say {\it $(\mathcal{M}, \evalit)$ 
	models $F$} iff  
	$\mathcal{M}, \evalit \models F$.}  
	 \hfill$\spadesuit$ 
\end{definition} 
\section{Equivalence-Equipped Dung's Model and Extensions} \label{sec_equivalence_equipped_dung_model_extensions}
In this section, we study an annotated-graph generalisation of Dung's argumentation 
model~\cite{Dung95}, accommodating an equivalence relation among graph nodes, 
and define a {\it novel set of 
extensions} representing acceptable nodes under certain criteria.
But first, Dung's model \cite{Dung95} 
is: $\mathcal{M}_{\mathsf{dung}} \equiv ((ObjStmts,ObjE), Obj\Pi)$  
with: $\{\} = Obj\Pi(u)$ for every $u \in ObjStmts$; 
$\{attacks\} = Obj\Pi((u_1,u_2))$ for every $(u_1,u_2) \in ObjE$.   
Several kinds of {\it extensions} defined by Dung 
are grounded in 
{\it conflict-freeness} and {\it defence}. 
\begin{itemize} 
	\item {\it Conflict-freeness}: A subset $ObjStmts'$ of $ObjStmts$ 
		is conflict-free iff there is no edge among them with $attacks$ annotation.  
	\item {\it Defence}: A subset $ObjStmts'$ of $ObjStmts$ 
		defends $u \in ObjStmts$ iff, for any $u' \in ObjStmts$, 
		if there is an $attacks$ annotated edge from $u'$ to $u$,  
		there is some $u'' \in ObjStmts'$ with an $attacks$ annotated edge from 
		$u''$ to $u'$.   
\end{itemize} 
A subset $ObjStmts'$ of $ObjStmts$ is called {\it admissible} 
iff it is conflict-free and defends every $u \in ObjStmts'$.   
An admissible subset $ObjStmts'$ is: 
\begin{itemize} 
	\item a {\it complete extension} iff it includes every graph 
		node it defends. 
	\item a {\it preferred extension} iff it is a maximal complete extension. 
	\item a {\it grounded extension} iff it is a minimal complete extension. 
	\item a {\it stable extension} iff it is a complete extension, and 
		for any $u' \in (ObjStmts \backslash\linebreak ObjStmts')$, there is some 
		$u'' \in ObjStmts'$ such that $(u'', u') \in ObjE$. 
\end{itemize} 

\noindent \textbf{Potential issues of conflict-freeness and defence.}   
$\mathcal{M}_{\mathsf{dung}}$ treats every graph node distinctly (because $ObjStmts$ is a set). 
In practice, utterances in a dialogue may not be distinct. 
Even if they are, 
the context of dialogue progression 
may treat them equally. 
Thus, if there is an annotation guideline by which utterances with the same content 
will be given a unique utterance ID, the same ID may be assigned to multiple utterances. 

With the discussion graph semantics accommodating annotations, 
we can express a more general argumentation model   
that admits an equivalence relation among graph nodes: $\mathcal{M}_{\sim} \equiv 
((ObjStmts, ObjE), Obj\Pi)$ with: $|Obj\Pi(u)| =1$ for every $u \in ObjStmts$; and $\{attacks\} = Obj\Pi((u_1, u_2))$ 
for every $(u_1, u_2) \in ObjE$. The annotation on $u \in ObjStmts$ 
is meant to be its utterance ID.

In this generalised model $\mathcal{M}_{\sim}$, both conflict-freeness and defence  
of $\mathcal{M}_{\mathsf{dung}}$ 
face a potential issue. To illustrate them, let's look at the annotated graph in the top 
part of Fig. \ref{fig:counterexample} in which 
every graph node is distinct (because $ObjStmts$ is a set), but where 
$u_3$ and $u_5$ are equivalent.  

\begin{itemize} 
	\item {\it Possible failure of conflict-freeness}: 
		$u_3$ and $u_4$ 
		form a conflict-free set. 
		However, they are not quite conflict-free when we heed the equivalence,  
		as $u_4$ has an $attacks$ annotated edge into $u_5$ 
		equivalent to $u_3$.  
	\item {\it Possible failure of defence}:  
		$\{u_1\}$ defends $u_3$.  
		However, $\{u_1\}$ does not defend $u_5$ that is equivalent to $u_3$.  
		Thus, if the equivalence is taken into account, 
		this defence is partial and incomplete. 
\end{itemize} 
\begin{figure}[!t]  
		\begin{center} 
	\begin{tikzcd}  
		u_1:\{ID_1\}  \arrow[r]{}{\{attacks\}} & u_2:\{ID_2\} \arrow[r]{}{\{attakcs\}} & u_3: \{ID_3\}
		& u_4:\{ID_4\} \arrow[r]{}{\{attacks\}} & u_5:\{ID_3\}\\
	\end{tikzcd}  
			\hrule 
		\end{center} 

		\begin{center}  
	\begin{tikzcd}  
		u_1:\{\}  \arrow[r]{}{\{attacks\}} & u_2:\{\} \arrow[r]{}{\{attakcs\}} & u_3': \{\}
		& u_4:\{\} \arrow[l]{}{\{attacks\}} \\ 
	\end{tikzcd}  
		\end{center}

		 \caption{\textbf{Top}: An example of an annotated graph 
		 with some equivalent graph nodes. 
		 \textbf{Bottom}: An annotated graph derived from 
		 collapsing the above annotated graph for equivalent nodes $u_3$ and $u_5$ 
		 into $u_3'$. 
		 } 
		 \label{fig:counterexample} 
\end{figure} 
Reduction of $\mathcal{M}_{\sim}$ into 
$\mathcal{M}_{\mathsf{dung}}$ by merging equivalent graph nodes into one 
(see the bottom part of Fig. \ref{fig:counterexample}) is acceptable in certain situations. 
But it generally results in an over-approximation and a loss of information.  
To better control the effect of equivalence, we 
introduce the following distinctions:

\begin{itemize} 
	\item {\it Simple-conflict-freeness}: Defined as conflict-freeness in Dung's model 
		(modulo annotations on nodes). 
	\item {\it Wide-conflict-freeness}: a subset $ObjStmts'$ of $ObjStmts$ is wide-conflict-free 
		iff there is no edge among cl$_{\sim}(ObjStmts')$, where 
		cl$_{\sim}(ObjStmts')$ is the closure of $ObjStmts'$ under node equivalence. 
		{\it Example:} $\text{cl}_{\sim}(\{u_3\}) = \text{cl}_{\sim}(\{u_5\}) = \{u_3, u_5\}$   
		in the top annotated graph of Fig. \ref{fig:counterexample}. 
	\item {\it Simple-defence}: Defined as defence in Dung's model (modulo 
		annotations on nodes). 
	\item {\it Wide-defence}: a subset $ObjStmts'$ of $ObjStmts$ wide-defends 
		$u \in ObjStmts$ iff $ObjStmts'$ simple-defends 
		each member of $\text{cl}_{\sim}(\{u\})$. 
\end{itemize} 
\begin{definition}[Simple and wide admissibilities]\label{def_simple_wide_admissible} \rm
     Given $\mathcal{M}_{\sim} \equiv ((ObjStmts,\linebreak ObjE), Obj\Pi)$,  
	we say 
	$ObjStmts' \subseteq ObjStmts$ is: {\it simple-admissible} iff it is simple-conflict-free and 
	it simple-defends every $u \in ObjStmts'$;  
	{\it wide-admissible} iff it is wide-conflict-free and it wide-defends every
	$u \in ObjStmts'$.  \hfill$\spadesuit$ 
\end{definition} 

\begin{proposition}[Existence of simple and wide admissible sets] 
 For $\mathcal{M}_{\sim} \equiv ((ObjStmts,ObjE), Obj\Pi)$ 
	and for  $\sigma \in \{\text{simple}, \text{wide}\}$, 
	there is $ObjStmts' \subseteq ObjStmts$ 
	such that $ObjStmts'$ is $\sigma$-admissible. 
\end{proposition}  
\begin{proof}  
	$\{ \}$ is both simple and wide admissible.  \hfill$\Box$ 
\end{proof} 

\begin{proposition}[Relationship between simple and wide admissibilities] \label{prop_relationship_simple_wide_admissibilities} 
For $\mathcal{M}_{\sim} \equiv ((ObjStmts,ObjE), Obj\Pi)$ 
and for $ObjStmts' \subseteq ObjStmts$, if $ObjStmts'$ is wide-admissible, then 
$ObjStmts'$ is simple-admissible. 
\end{proposition}   
\begin{proof} 
	$ObjStmts' \subseteq \text{cl}_{\sim}(ObjStmts')$. Vacuous. \hfill$\Box$ 
\end{proof} 

The converse---that simple-admissibility
implies wide-admissibility---does not always hold. We will see 
that in Example \ref{example_illustration_sigamtau_complete_extensions}. 
But for any 
$\mathcal{M}_{\mathsf{dung}}$, the two versions of admissibilities are trivially indistinguishable. \\

\noindent \textbf{Generalised extensions.} 
We now generalise Dung's `extensions'. 
In $\mathcal{M}_{\mathsf{dung}}$, a complete extension imposes an extra condition 
that a (simple) admissible
set of graph nodes includes every node it (simple) defends. This {\it closure by defence} 
is not consistent with the wide-admissibility, so we  
revise it as follows. 
\begin{definition}[Closure by defence]  \rm 
   Given $\mathcal{M}_{\sim} \equiv 
	((ObjStmts,ObjE),\linebreak Obj\Pi)$ and $ObjStmts' \subseteq ObjStmts$, 
	we say $ObjStmts'$ is: 
	closed under simple (resp. wide) defence iff it includes every graph node it 
	simple- (resp. wide-) defends.  \hfill$\spadesuit$ 
\end{definition}  
This closure does not concern closure by equivalence which we 
define anew. 
\begin{definition}[Closure by equivalence] \rm 
     Given $\mathcal{M}_{\sim} \equiv 
	((ObjStmts,\linebreak ObjE),Obj\Pi)$ and $ObjStmts' \subseteq ObjStmts$, 
	we say $ObjStmts'$ is: {\it closed under equivalence} iff 
	$ObjStmts' = \text{cl}_{\sim}(ObjStmts')$.  \hfill$\spadesuit$ 
\end{definition} 
These closures give rise to the following classifications of 
$\mathcal{M}_{\sim}$ extensions.  
\begin{definition}[Extensions for $\mathcal{M}_{\sim}$] \rm 
     Given $\mathcal{M}_{\sim} \equiv 
	((ObjStmts,ObjE),Obj\Pi)$ and $ObjStmts' \subseteq ObjStmts$,  
	let $\sigma$ be one of $\{\text{simple},\text{wide}\}$ and 
	let $\tau$ be one of $\{\sigma \text{defence}, \text{equivalence}\}$.  
	Let $ObjStmts'$ be $\sigma$-admissible, 
	we say $ObjStmts'$ is: 
	\begin{itemize} 
		\item $\sigma\tau$-complete iff 
		      it is closed under $\tau$. 
	        \item $\sigma$-complete iff  
		      it is closed under both $\sigma$defence and 
			equivalence. 
	        \item $\sigma\tau$-preferred iff 
		      it is maximally $\sigma\tau$-complete.
	      	\item $\sigma$-preferred iff 
		      it is maximally $\sigma$-complete.
	      	\item $\sigma\tau$-grounded iff 
		      it is minimally $\sigma\tau$-complete. 
	        \item $\sigma$-grounded iff 
		      it is minimally $\sigma$-complete. 
	        \item $\sigma\tau$-stable iff  
		      it is $\sigma\tau$-complete, and for any 
			$u \in ObjStmts \backslash ObjStmts'$, 
			there is some $u' \in ObjStmts'$ such that 
			$(u', u) \in ObjE$. 
		\item $\sigma$-stable iff   
		      it is $\sigma$-complete, and for any 
			$u \in ObjStmts \backslash ObjStmts'$, 
			there is some $u' \in ObjStmts'$ such that 
			$(u', u) \in ObjE$.  \hfill$\spadesuit$ 
	\end{itemize} 
\end{definition} 

In case $\mathcal{M}_{\sim}$ is a Dung's model (modulo annotations on nodes), a $\sigma\tau$-complete extension is just a simple complete extension; similarly 
for all the other types of extensions. However, in general, 
most of $\sigma\tau$-x extensions are distinct notions.  
For brevity, we may write {\it simple defence-complete} to mean 
{\it simple simpledefence-complete}, {\it wide defence-complete} 
to mean {\it wide widedefence-complete} and so on. 

\begin{figure}[!t]  
		\begin{center} 
	\begin{tikzcd}  
		u_1:\{ID_1\}  \arrow[r]{}{\{attacks\}} & u_2:\{ID_2\} \arrow[r]{}{\{attakcs\}} & u_3: \{ID_3\}
		& u_4:\{ID_4\} \arrow[r]{}{\{attacks\}} & u_5:\{ID_3\}\\
	\end{tikzcd}  
			\hrule 
		\end{center} 

		\begin{center}  
	\begin{tikzcd}  
		u_1:\{ID_1\} 
		\arrow[r]{}{\{attacks\}} & u_2:\{ID_2\} \arrow[r]{}{\{attakcs\}} & u_3: \{ID_3\} \\
		u_4:\{ID_1\} \arrow[r]{}{\{attacks\}} & u_5:\{ID_4\} \arrow[l]{}{\{attacks\}} 
		\arrow[r]{}{\{attacks\}}
		& u_6:\{ID_3\}
	\end{tikzcd}  
		\end{center} 
		 \caption{\textbf{Top}: An example of an object-level annotated graph 
		 with equivalent graph nodes.  
		 \textbf{Bottom}: Another example of an object-level annotated graph 
		 with equivalent graph nodes.} 
		 \label{fig_illustration_complete_extension_distinction} 
\end{figure} 

\begin{example}[Illustration of $\sigma\tau$-complete extensions] \label{example_illustration_sigamtau_complete_extensions} 
    Look at Fig. \ref{fig_illustration_complete_extension_distinction}. 
	For the top annotated graph, $\{ \}, \{u_1\}, \{u_4\}, \{u_1, u_4\}$ are wide-admissible 
	(and therefore also simple-admissible), 
	and additionally $\{u_1, u_3\}, \{u_1, u_3, u_4\}$ are simple-admissible. 
    \begin{itemize}  
	    \item Simple defence-complete extensions: $\{u_1, u_3, u_4\}$. 
	    \item Wide defence-complete extensions: $\{u_1, u_4\}$. 
	    \item Simple equivalence-complete extensions: $\{ \}, \{u_1\}, \{u_4\}, \{u_1, u_4\}$. 
	    \item Wide equivalence-complete extensions: $\{ \}, \{u_1\}, \{u_4\}, \{u_1, u_4\}$.  
	    \item Simple complete extensions: (does not exist) 
	    \item Wide complete extensions: $\{u_1,u_4\}$. 
    \end{itemize}  
	For the bottom, $\{ \}, \{u_4\}, \{u_5\}, \{u_1, u_4\}, 
	\{u_1, u_3, u_4\}, \{u_1, u_4, u_6\}, \{u_1, u_3, u_4, u_6\}$ 
	are wide-admissible (and simple-admissible). Additionally, 
	$\{u_1\}, \{u_1, u_3\}, \{u_4, u_6\},\linebreak \{u_1, u_5\}, \{u_1, u_3, u_5\}$ 
	are simple-admissible. 
	\begin{itemize} 
		\item Simple defence-complete extensions: $\{u_1, u_3\},  
			\{u_1, u_3, u_5\}, \{u_1, u_3, u_4, u_6\}$. 
		\item Wide defence-complete extensions:  
			$\{ \}, \{u_4\}, \{u_5\}, \{u_1, u_3, u_4, u_6\}$. 
		\item Simple equivalence-complete extensions: $\{ \}, \{u_5\}, 
			\{u_1,u_4\},  \{u_1,u_3,u_4,u_6\}$.
		 \item Wide equivalence-complete extensions: $\{  \},    
			 \{u_5\}, \{u_1, u_4\}, \{u_1,u_3,u_4,u_6\}$. 
		 \item Simple complete extensions: $\{u_1, u_3, u_4, u_6\}$. 
		 \item Wide complete extensions: $\{ \}, \{u_5\}, \{u_1,u_3,u_4,u_6\}$.  
			 \hfill$\clubsuit$ 
	\end{itemize} 
\end{example}  
In Example \ref{example_illustration_sigamtau_complete_extensions}, the simple and wide 
equivalence-complete 
extensions coincide. This is no coincidence. 
\begin{theorem}[Collapse of simple and wide equivalence closures] \label{thm_collapse} 
For $\mathcal{M}_{\sim} \equiv 
	((ObjStmts,ObjE),Obj\Pi)$, $ObjStmts' \subseteq ObjStmts$ 
	and $\mu \in \{\text{complete},\linebreak \text{preferred, grounded, stable}\}$, 
	$ObjStmts'$ is a simple equivalence-$\mu$ extension iff\linebreak
		 $ObjStmts'$ is a wide equivalence-$\mu$ extension.  
\end{theorem}  
\begin{proof} 
	First, consider $\mu =$ complete. Suppose $ObjStmts'$ is a 
	simple equivalence-complete extension. Because $ObjStmts'$ is simple-admissible, 
	it is simple-conflict-free. But also $ObjStmts' = \text{cl}_{\sim}(ObjStmts')$ 
	because $ObjStmts'$ satisfies closure by equivalence, 
	so $ObjStmts'$ is also wide-conflict-free. 
	$ObjStmts'$ includes every graph node it simple-defends. But 
	$ObjStmts' = \text{cl}_{\sim}(ObjStmts')$, so it includes every graph node 
	it wide-defends. Hence, $ObjStmts'$ is a wide equivalence-complete extension. 

	Into the other direction, suppose $ObjStmts'$ is a wide equivalence-complete extension.  
	Because $ObjStmts'$ is wide-admissible, it is wide-conflict-free. 
Because $ObjStmts' = \text{cl}_{\sim}(ObjStmts')$, it is also simple-conflict-free.  
	$ObjStmts'$ includes every graph node it wide-defends. Again, 
	$ObjStmts' = \text{cl}_{\sim}(ObjStmts')$, so it includes every graph node it simple-defends, too.   

	Now, consider $\mu =$ preferred. Since the set of simple equivalence-complete extensions 
	coincides with that of wide equivalence-complete extensions, 
	by definition, that of simple equivalence-preferred extensions coincides with 
	that of wide equivalence-preferred extensions. Similarly for $\mu =$ grounded. 
	
	Finally, consider $\mu =$ stable. Suppose $ObjStmts'$ is a 
	simple equivalence-stable extension. It is simple-conflict-free 
	and for any node not in $ObjStmts'$, some member of $ObjStmts'$ has an $attacks$ annotated 
	edge to it. Further,\linebreak $ObjStmts' = \text{cl}_{\sim}(ObjStmts')$. So, $ObjStmts'$ 
	is wide-conflict-free. Closure by defence is obvious. 
	The other direction is vacuous. 
	\hfill$\Box$ 

\end{proof}

As is clear from the bottom annotated graph in Example \ref{example_illustration_sigamtau_complete_extensions}, 
there may exist multiple minimal wide defence-complete extensions and hence multiple 
wide defence-grounded extensions. Recall that  
the grounded extension is the least fixpoint (lfp) of a function for 
each $\mathcal{M}_{\mathsf{dung}}$. It turns out the least fixpoint characterisation 
of wide defence-grounded extensions is similarly possible for $\mathcal{M}_{\sim}$ 
albeit with multiple functions. 
\begin{theorem}[Characterisation of wide defence-grounded extensions]  \label{thm_characterisation_wide_defence_grounded_extensions} 
	Given $\mathcal{M}_{\sim} \equiv ((ObjStmts, ObjE), Obj\Pi)$, 
	let $\mathfrak{F}$ be the set of functions 
	$\pmb{F}: \mathfrak{p}(ObjStmts) \rightarrow \mathfrak{p}(ObjStmts)$ 
	with the following conditions for $ObjStmts' \subseteq ObjStmts$. 
	\begin{itemize} 
		\item If $ObjStmts'$ is a wide defence-complete extension, 
			then $\pmb{F}(ObjStmts') = ObjStmts'$. 
		\item Otherwise, 
			\begin{itemize} 
				\item $ObjStmts' \subset \pmb{F}(ObjStmts')$, and 
		 		\item $ObjStmts'$ wide-defends 
			each $u \in (\pmb{F}(ObjStmts') \backslash ObjStmts')$. 
			\end{itemize} 
	\end{itemize}   
	Let $\Gamma$ be the set containing 
	all minimal members of $\bigcup_{\pmb{F} \in \mathfrak{F}}\text{lfp}(\pmb{F})$.  
	For $ObjStmts' \subseteq ObjStmts$, it holds that 
	$ObjStmts' \in \Gamma$ iff $ObjStmts'$ is a wide defence-grounded extension. 
\end{theorem}  
\begin{proof} 
     $\{ \}$ is wide-admissible. If it is a wide defence-complete extension,  
	then it is a wide defence-grounded extension, since there is no 
	smaller subset of $ObjStmts$ than $\{ \}$.  
	Otherwise, we have a non-empty subset $ObjStmts'' \equiv 
	\pmb{F}(\{ \})$. By definition, $\{ \}$ wide-defends every member $u$ of $ObjStmts''$. 
	Now, suppose, by showing contradiction, $ObjStmts''$ is not wide-conflict-free. 
	Then there is some $u_1, u_2 \in ObjStmts''$, 
	some $u_1' \in \text{cl}_{\sim}(\{u_1\})$ and 
	some $u_2' \in \text{cl}_{\sim}(\{u_2\})$ such that  
	$(u_1', u_2') \in ObjE$. However, since $\{ \}$ wide-defends $u_1$ and $u_2$,  
	$\{ \}$ simple-defends $u_1'$ and $u_2'$. Thus, 
	neither $u_1'$ nor $u_2'$ has an $attacks$ annotated incoming edge, contradiction.  
	$\pmb{F}(\{ \})$ is therefore wide-admissible. 
	In general, for any wide-admissible subset $ObjStmts''$, 
	$\pmb{F}(ObjStmts'')$ is shown wide-admissible.   
	Since $ObjStmts$ is finite, every $\pmb{F}$ has its least fixpoint. 
	Since $\mathfrak{F}$ includes every possible $\pmb{F}$, each minimal 
	member of $\bigcup_{\pmb{F} \in \mathfrak{F}}\text{lfp}(\pmb{F})$ is by definition 
	a wide defence-grounded extension. \hfill$\Box$ 
\end{proof} 
\noindent Concerning the existence of $\mathcal{M}_{\sim}$ extensions, we have the following result. 

\begin{theorem}[On Existence]
	 Given $\mathcal{M}_{\sim}$,    
	 $\sigma \in \{\text{simple}, \text{wide}\}$, 
	 $\tau \in \{\sigma\text{defence},\linebreak \text{equivalence}\}$ 
	and $\mu \in \{\text{complete}, \text{preferred}, \text{grounded}\}$,  
	the properties below hold. 
	\begin{itemize} 
		\item  $\mathcal{M}_{\sim}$ has at least one $\sigma\tau$-$\mu$ extension.  
		\item  $\mathcal{M}_{\sim}$ has at least one wide-$\mu$ extension. 
	\end{itemize} 
	\end{theorem}  
\begin{proof} 
     $\{ \}$ is both simple and wide admissible. It is closed under node equivalence. 
	Hence, $\{ \}$ is a $\sigma$ equivalence-complete extension. 
	Since $ObjStmts$ is a finite set and $\{ \}$ is the least 
	subset of it, it follows that $\mathcal{M}_{\sim}$ has at least 
	one $\sigma$ equivalence-preferred extension and at least one $\sigma$ equivalence-grounded 
	extension. 

	For the $\sigma$ $\sigma$defence-$\mu$ extension, if $\sigma$ = simple, 
	the existence is immediate from \cite{Dung95}. If $\sigma$ = wide, 
	the existence is immediate from 
	Theorem \ref{thm_characterisation_wide_defence_grounded_extensions} 
	and the facts that a wide defence-grounded extension is a wide defence-complete extension 
	and that $ObjStmts$ is a finite set. 
	
 	Since a wide-$\mu$ extension is wide-admissible, its existence 
	is immediate from the existence of both wide defence-$\mu$ and 
	wide equivalence-$\mu$ extensions, which we have shown above. 
	\hfill$\Box$ 
\end{proof} 
Simple-$\mu$ extensions may not exist (see Example~\ref{example_illustration_sigamtau_complete_extensions}) since closure by equivalence can be in conflict with 
simple admissibility.

\section{Reasoning} 
We can reason about discussion and argumentation graphs using the discussion graph semantics of
\textsf{FOL}. 
In this section, our main goal is to firstly show that  
all $\mathcal{M}_{\sim}$ extensions are 
first-order characterisabile, thereby connecting 
the contributions of Sections \ref{sec_discussion_graph_semantics_first_order_logic_equality} and \ref{sec_equivalence_equipped_dung_model_extensions}, and secondly to show that the set of all $\mathcal{M}_{\sim}$'s extensions of a give type 
is equally first-order characterisable. 
But since we introduced the idea of reasoning about pattern detection at the beginning, 
we start with an illustrative example of that task.  
\begin{example}[Pattern detection]  
	As we saw (in Fig. \ref{fig_toulmin}), Toulmin's model requires 
	 the presence of specific annotations on object-level statements 
	 while the statements themselves are some texts. 
	      \begin{center}  
		      \adjustbox{scale=0.9}
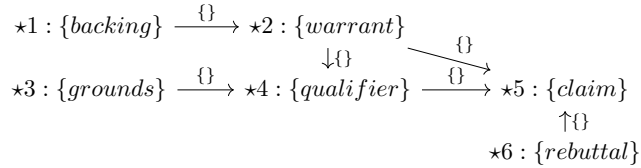
{
	      \begin{tikzcd}[row sep=small] 
		       |[alias=A1]| \star 1:\{backing\} \arrow[r]{}{\{\}} & 
		      \star 2:\{warrant\} \arrow[dr]{}{\{\}}
		      \arrow[d]{}{\{\}}\\ 
		      |[alias=A3]|	\star 3:\{grounds\} \arrow[r]{}{\{\}} & 
		      \star 4:\{qualifier\} \arrow[r]{}{\{\}} & 
		      \star 5:\{claim\} \\ 
		       & & |[alias=A2]| \star 6:\{rebuttal\} \arrow[u,swap]{}{\{\}}
	\end{tikzcd} 
		} 
	\end{center}      
        Let $\evalit$ and a predicate symbol $p$ be such that  
	$\evalit(p)$ denotes the above skeleton annotated graph.  
	Then, there is a Toulmin's model in $((ObjStmts, ObjE), Obj\Pi)$ 
	iff $((ObjStmts, ObjE), Obj\Pi, \evalit)$ models 
	    $\exists x_1.\ldots. 
	    \exists x_6.p(x_1, \ldots, x_6)$. 
\hfill$\clubsuit$ 
\hide{ 
	One example of argumentation in Dung's model is shown below. 
	\begin{center} 
	      \begin{tikzcd}[row sep=small,/tikz/execute at end picture={
			      \node (large) [rectangle, draw, fit=(A0) (A1) (A2) (A3) (A10),
			      label=above:Dung's model \cite{Dung95}] {}; 
  }]    
		      & & |[alias=A0]|  \\
		       |[alias=A1]| s_1:* \arrow[r,bend right=15,swap]{}{attacks}
		      & s_2:* \arrow[l, shift left, bend right=15,swap]{}{attacks} 
		      \arrow[r]{}{attacks} 
		      & s_3:* \arrow[r]{}{attacks} & s_4:* \arrow[r]{}{attacks} & 
		      |[alias=A2]| 
		      s_5:* \arrow[ll, bend right=30,swap]{}{attacks} \\
		      & & |[alias=A10]| 
	\end{tikzcd} 
	\end{center}   
	One example of discussion in Kunz {\it et al.}'s model is shown below. 
	\begin{center} 
		\begin{tikzcd}[row sep=small,/tikz/execute at end picture={
			      \node (large) [rectangle, draw, fit=(A1) (A2) (A3),
			      label=above:Kunz {\it et al.}'s model \cite{Kunz70} 
			      ] {}; 
  }] 
		      & |[alias=A1]| s_1:position \arrow[r]{}{responds-to} & s_2:issue 
			\arrow[dr]{}{questions}\\ 
		      |[alias=A3]|	s_3:issue \arrow[r,swap]{}{questions} & 
		      s_4:argument \arrow[r,swap]{}{supports} & s_5:position \arrow[r,swap]{}{responds-to}
			& |[alias=A2]| s_6:issue 
	\end{tikzcd} 
	\end{center}   
}
\end{example} 
We now show 
the first-order characterisability of $\mathcal{M}_{\sim}$'s extensions. 
The closest prior work is \cite{Doutre14} which 
established the propositional characterisability of $\mathcal{M}_{\mathsf{dung}}$'s 
admissible sets, as well as its complete and stable extensions. However, 
\cite{Doutre14} did not provide a full propositional  
characterisation of all $\mathcal{M}_{\mathsf{dung}}$'s extensions. 
In particular, both the preferred and grounded 
extensions—each requiring comparisons among complete extensions—were left open, 
with a note suggesting that they might be characterisable in propositional {\it dynamic} logic. 
The question thus remains: {\it what about their propositional characterisability?} 
By establishing a complete first-order characterisation of 
$\mathcal{M}_{\sim}$'s extensions, 
we address this question en route.

As preparations, we define formulas 
$k\text{-}\mathsf{CF}(t_1, t_2, \ldots, t_k)$, 
$kN\text{-}\mathsf{WCF}(t_1, t_2, \ldots, t_k)$, 
$k\text{-}\mathsf{DF}(t, t_1, t_2, \ldots, t_k)$ 
and $kN\text{-}\mathsf{WDF}(t, t_1, t_2, \ldots, t_k)$ for $0 \leq k \leq N$, 
and also $kl\text{-}\mathsf{CL}(t_1, \ldots, t_{k}, \ldots, t_{l})$ 
for $0 \leq k \leq l$.  
$\bigvee_{i \leq k} F_i$ is shorthand for 
$F_1 \vee \cdots \vee F_k$. Similarly 
for $\bigwedge_{i \leq k} F_i$.  
More generally, we use $\bigvee_{\textsf{condition}} F$   
and $\bigwedge_{\textsf{condition}} F$, 
and when \textsf{condition} is not met, 
$\bigvee_{\textsf{condition}} F$ is  $\bot$ and 
			$\bigwedge_{\textsf{condition}} F$ is $\top$. 

\begin{definition}[k-CF]\label{def_cf} \rm   
	\noindent {\small $k\text{-}\mathsf{CF}(t_1, t_2, \ldots, t_k)$} is:
	\begin{itemize} 
		\item {\small $\top$} if $0 = k$. 
		\item {\small $p_{\mathsf{D}}(t_1, t_2, \ldots, t_k) 
			\wedge \forall y_1.\forall y_2.(p_{\mathsf{D}}(y_1) 
			\wedge p_{\mathsf{D}}(y_2) 
			\wedge (\bigvee_{i \leq k}y_1 = t_i)
			\wedge (\bigvee_{i \leq k}y_2 = t_i) 
			\supset (\neg y_1 = y_2 \wedge \neg p_{\mathsf{A}}(y_1, y_2) \wedge 
			\neg p_{\mathsf{A}}(y_2, y_1)) \vee 
			(y_1 = y_2 \wedge \neg p_A(y_1)
			))$} if $1 \leq k$. \hfill$\spadesuit$ 
			\end{itemize} 
\end{definition}   
$p_{\mathsf{D}}$ is for judging 
whether (graph) nodes are distinct, and $p_{\mathsf{A}}$ is for judging 
whether a node attacks another (judging self-attack if the arity is 1). 
So, 
$k\text{-}\mathsf{CF}(t_1,\linebreak \ldots, t_k)$ is for judging 
whether $k$ distinct nodes form a simple-conflict-free set.  

In the remainder of this paper, we write  
{\it X  iff$^*$ 
$\mathcal{M}_{\sim}, (\mu, \mathcal{I}) \models F$} to mean  
{\it X iff there is some variable assignment $\mu'$ such that 
$\mathcal{M}_{\sim}, (\mu', \mathcal{I}) \models F$}.  

\begin{lemma}[Characterisability of a simple-conflict-free set] \label{lem_k_cf} 
     Given {\small $\mathcal{M}_{\sim} \equiv \linebreak
	((ObjStmts,ObjE), Obj\Pi)$}, 
 	let $\evalit$ be such that: $\evalit(p_{\mathsf{A}})$ is 
	\resizebox{3cm}{!}{$[\star 1:\{\} \xrightarrow{\{attacks\}} \star 2:\{\}]$} if $p_{\mathsf{A}}$'s arity is 2 and   
\resizebox{3cm}{!}{$
[\! 
\begin{tikzpicture}[baseline=(n.base)]
  \node (n) {$\star 1:\{\}$};
  \draw[->, looseness=8, out=10, in=-10]
    (n) to node[right, xshift=2pt] {\scriptsize$\{attacks\}$} (n);
\end{tikzpicture}
\!]
$}
	if it is 1; and $\evalit(p_{\mathsf{D}})$ 
	is 
	\resizebox{3cm}{!}{$[\star 1:\{\} \ \star 2:\{\} \ \cdots \  \star n:\{\}] $} 
	with $n$ being the arity of $p_{\mathsf{D}}$. 
	Let $k$ be $0 \leq k \leq |ObjStmts|$. 
	Then $\{\evalit(c_1), \ldots,\linebreak \evalit(c_k)\}$ is 
	simple-conflict-free iff$^*$ $\mathcal{M}_{\sim}, \evalit 
	\models k\text{-}\mathsf{CF}(c_1, \ldots, c_k)$. 
\end{lemma} 
\begin{proof} 
   Let $\overrightarrow{c}$ denote 
	$c_1, \ldots, c_k$.  
	
	Suppose $0 < k$. 
	First, $\{\evalit(c_1), \ldots, \evalit(c_k)\}$ 
	is trivially not a simple-conflict-free set in $\mathcal{M}_{\sim}$ 
	when either of the following conditions holds. 
	\begin{itemize} 
		\item There is a pair of natural numbers $i$ and $j$ 
			such that $1 \leq i < j \leq k$ and that 
			$\evalit(c_i) = \evalit(c_j)$. In this case, 
			$\{\evalit(c_1), \ldots, \evalit(c_k)\}$ 
			is not a set. 
		\item There is a natural number $i$ such that 
			$1 \leq i \leq k$ and that $\evalit(c_i)$ 
			is not a member of $ObjStmts$.  
			$\{\evalit(c_1), \ldots, \evalit(c_k)\}$ 
			is not a subset of $ObjStmts$. 
	\end{itemize}  
	In these cases, $\mathcal{M}_{\sim}, \evalit \not\models 
	p_{\mathsf{D}}(\overrightarrow{c})$ ({\it i.e.} it is not the case that  
$\mathcal{M}_{\sim}, \evalit \models 
	p_{\mathsf{D}}(\overrightarrow{c})$), and so 
	$\mathcal{M}_{\sim}, \evalit \not\models  
	k\text{-}\mathsf{CF}(\overrightarrow{c})$ 
	
	Now, let $\{\evalit(c_1), \ldots, \evalit(c_k)\}$ ($0 < k$) be a subset 
	of $ObjStmts$, which we denote by $ObjStmts'$, with $|ObjStmts'| = k$. 
	We prove that $ObjStmts'$ is a simple conflict-free set  
	iff$^*$ $\mathcal{M}_{\sim}, \evalit \models   
	p_{\mathsf{D}}(\overrightarrow{c}) 
			\wedge \forall y_1.\forall y_2.(p_{\mathsf{D}}(y_1) 
			\wedge p_{\mathsf{D}}(y_2) 
			\wedge (\bigvee_{i \leq k}y_1 = c_i)
			\wedge (\bigvee_{i \leq k}y_2 = c_i) 
			\supset (\neg y_1 = y_2 \wedge \neg p_{\mathsf{A}}(y_1, y_2) \wedge 
			\neg p_{\mathsf{A}}(y_2, y_1)) \vee 
			(y_1 = y_2 \wedge \neg p_A(y_1)
			))$. 
        
	\textbf{Only if:} 
        $\mathcal{M}_{\sim}, \evalit \models p_{\mathsf{D}}(\overrightarrow{c})$ 
       holds by assumption. Let $\evalit'$ be almost exactly $\evalit$ except that $\evalit'(y_1)$ and/or $\evalit'(y_2)$ 
	may be different from $\evalit(y_1)$ and/or $\evalit(y_2)$.  
       \begin{description} 
	       \item[\textbf{Case $\evalit'(y_1) \not\in ObjStmts$ OR $\evalit'(y_2) \not\in ObjStmts$}:]  
		       $\mathcal{M}_{\sim}, \evalit' \not\models p_{\mathsf{D}}(y_1) \wedge p_{\mathsf{D}}(y_2)$. 
		\item[\textbf{Case $\evalit'(y_1), \evalit'(y_2) \in ObjStmts$}:]{\ } 
			\begin{description} 
				\item[\textbf{Sub-case $\evalit'(y_1) \not\in ObjStmts'$ OR $\evalit'(y_2) \not\in ObjStmts'$}:]  
					$\mathcal{M}_{\sim}, \evalit' \not\models (\bigvee_{i \leq k} y_1 = c_i) 
					\wedge (\bigvee_{i \leq k} y_2 = c_i)$. 
				\item[\textbf{Sub-case $\evalit'(y_1), \evalit'(y_2) \in ObjStmts'$}:]  
					By assumption, no member of $ObjStmts'$ attacks a member of $ObjStmts'$ including itself. Therefore, 
					$\mathcal{M}_{\sim}, \evalit' \models 
					(\neg y_1 = y_2 \wedge \neg p_{\mathsf{A}}(y_1, y_2) \wedge 
					\neg p_{\mathsf{A}}(y_2,\linebreak y_1)) \vee (y_1 = y_2 \wedge \neg p_A(y_1))$. 
			\end{description} 
       \end{description} 
	Hence, $\mathcal{M}_{\sim}, \evalit \models p_{\mathsf{D}}(\overrightarrow{c}) \wedge 
       \forall y_1.\forall y_2.(p_{\mathsf{D}}(y_1) 
			\wedge p_{\mathsf{D}}(y_2) 
			\wedge (\bigvee_{i \leq k}y_1 = c_i)
			\wedge (\bigvee_{i \leq k}y_2 = c_i) 
			\supset (\neg y_1 = y_2 \wedge \neg p_{\mathsf{A}}(y_1, y_2) \wedge 
			\neg p_{\mathsf{A}}(y_2, y_1)) \vee 
			(y_1 = y_2 \wedge \neg p_A(y_1)
			))$, as required.  \\

\textbf{If}: Let $\evalit'$ be such that 
$\mathcal{M}_{\sim}, \evalit' \models 
	k\text{-}\mathsf{CF}(\overrightarrow{c})$. We have to 
	prove that, for every $\evalit$ which is almost exactly $\evalit'$ 
	except for the variable assignment, 
	$\{\evalit(c_1), \ldots, \evalit(c_k)\}$ ($0 < k$) 
	is a simple-conflict-free set. 
	But $k\text{-}\mathsf{CF}(\overrightarrow{c})$ 
	is a well-formed formula, so the difference in variable assignment makes no 
	difference in satisfiability judgement. Hence, it suffices to prove that 
	$\{\evalit'(c_1), \ldots, \evalit'(c_k)\}$ is a simple-conflict-free set. 

	First, we prove that 
	$\{\evalit'(c_1), \ldots, \evalit'(c_k)\}$ is indeed a subset of $ObjStmts$ 
	with size $k$. But $\mathcal{M}_{\sim}, \evalit' \models k\text{-}\mathsf{CF}(\overrightarrow{c})$ and therefore $\mathcal{M}_{\sim}, \evalit' \models 
	 p_{\mathsf{D}}(\overrightarrow{c})$. Consequently, 
	 $\{\evalit'(c_1), \ldots, \evalit'(c_k)\}$ is a subset of $ObjStmts$ 
	 with size $k$, as required. 

	 Now, by assumption, 
	 for each $\evalit'(c_i)$,  
	 $(\evalit'(c_i), \evalit'(c_i)) \not\in ObjE$ holds, 
	 and for each pair $(\evalit'(c_i), \evalit'(c_j))$ 
	 with $1 \leq i \not= j \leq k$, we have 
	 $(\evalit'(c_i), \evalit'(c_j)) \not\in ObjE$, as required. \\

	 Suppose $0=k$. We must prove that $\{ \}$ is 
	 simple-conflict-free iff$^*$  
	 $\mathcal{M}_{\sim}, \evalit \models 0\text{-}\mathsf{CF}$ 
	 iff$^*$ $\mathcal{M}_{\sim}, \evalit \models \top$. 

\textbf{Only if}:  We have to prove 
that, for any $\evalit'$ which is almost exactly $\evalit$ except 
$\evalit'(x)$ may not be $\evalit(x)$, 
$\mathcal{M}_{\sim}, \evalit' \models \top$, which is vacuous. 

\textbf{If}: We have to prove that $\{ \}$ is simple-conflict-free. 
But this follows from definition of simple-conflict-freeness. \hfill$\Box$ 
\end{proof}

\begin{definition}[kl-CL]\label{def_cl} \rm   
	\noindent {\small $kl\text{-}\mathsf{CL}(t_1, \ldots, t_l)$} is:
	\begin{itemize} 
		\item {\small $\top$} if $0 = k = l$.  
		\item {\small $\bot$} if ($0 = k \not= l$) OR ($1 \leq k$ AND $l < k$).   
		\item {\small $p_{\mathsf{D}}(t_1, \ldots, t_l) 
			\wedge \forall z_1.\forall z_2.(p_{\mathsf{D}}(z_1,z_2) \wedge 
			(\bigvee_{i \leq k} z_1 = t_i) \wedge 
			\exists z_3.p_{\mathsf{AnnoEq}}(z_1,z_2,z_3) \supset (\bigvee_{i \leq l} z_2 = t_i)) \wedge \forall z_1.(p_{\mathsf{D}}(z_1) \wedge (\bigvee_{k < i \leq l} z_1 = t_i) \supset 
			\exists z_2.\exists z_3.((\bigvee_{i \leq k} z_2 = t_i) \wedge p_{\mathsf{AnnoEq}}(z_1, z_2, z_3))),$ } 
			otherwise.  \hfill$\spadesuit$ 
			\end{itemize} 
\end{definition}  
$p_{\mathsf{AnnoEq}}$ is for whether two distinct nodes share the same ID. So, 
$kl$-$\mathsf{CL}(t_1, \ldots, t_l)$ is for whether $l$ nodes form the closure of $k$ 
nodes under node equivalence. 

\begin{lemma}[Characterisability of closure by node 
	equivalence] \label{lem_kl_cl} 
    Given {\small $\mathcal{M}_{\sim} \equiv 
	((ObjStmts,ObjE), Obj\Pi)$},  
	let $k$ be no greater than $l$, and 
 	let $\evalit$ be such that: $\evalit(p_{\mathsf{D}})$ 
	is
	\resizebox{3cm}{!}{$[\star 1:\{\} \ \star 2:\{\} \ \cdots \  \star n:\{\}] $} 
	with $n$ being the arity of $p_{\mathsf{D}}$;   and 
	$\evalit(p_{\mathsf{AnnoEq}})$ is 
	\resizebox{3cm}{!}{$[\star 1:\{\star 3\} \ \star 2:\{\star 3\}]$}.  
	Then $\{\evalit(c_1), \ldots,\evalit(c_l)\}$ is  
	a subset\linebreak $ObjStmts'$ of $ObjStmts$ 
	with $ObjStmts' = \text{cl}_{\sim}(\{\evalit(c_1), \ldots, \evalit(c_k)\})$
	iff$^*$ $\mathcal{M}_{\sim},\linebreak \evalit 
	\models kl\text{-}\mathsf{Cl}(c_1, \ldots, c_l)$. 
\end{lemma} 
\begin{proof} 
   Let $\overrightarrow{c}$ denote 
	$c_1, \ldots, c_l$.  
	
	Suppose $0 < k$. 
	First, $\{\evalit(c_1), \ldots, \evalit(c_l)\}$ 
	is trivially not a subset of $ObjStmts$ in $\mathcal{M}_{\sim}$ 
	when either of the following conditions holds. 
	\begin{itemize} 
		\item There is a pair of natural numbers $i$ and $j$ 
			such that $1 \leq i < j \leq l$ and that 
			$\evalit(c_i) = \evalit(c_j)$. 
		\item There is a natural number $i$ such that 
			$1 \leq i \leq l$ and that $\evalit(c_i)$ 
			is not a member of $ObjStmts$.  
	\end{itemize}  
	In these cases, $\mathcal{M}_{\sim}, \evalit \not\models 
	p_{\mathsf{D}}(\overrightarrow{c})$, and so 
	$\mathcal{M}_{\sim}, \evalit \not\models  
	kl\text{-}\mathsf{CL}(\overrightarrow{c})$. 
	
	Now, let $\{\evalit(c_1), \ldots, \evalit(c_l)\}$ ($0 < k \leq l$) be a subset 
	of $ObjStmts$, which we denote by $ObjStmts'$, with $|ObjStmts'| = l$. 
	We prove that $ObjStmts' = \text{cl}_{\sim}
	(\{\evalit(c_1), \ldots, \evalit(c_k)\})$ 
	iff$^*$ $\mathcal{M}_{\sim}, \evalit \models    
	p_{\mathsf{D}}(\overrightarrow{c}) 
			\wedge \forall z_1.\forall z_2.(p_{\mathsf{D}}(z_1,z_2) \wedge 
			(\bigvee_{i \leq k} z_1 = c_i) \wedge 
			\exists z_3.p_{\mathsf{AnnoEq}}(z_1,z_2,z_3) \supset 
			(\bigvee_{i \leq l} z_2 = c_i)) \wedge \forall z_1.(p_{\mathsf{D}}(z_1) \wedge (\bigvee_{k < i \leq l} z_1 = c_i) \supset 
			\exists z_2.\exists z_3.((\bigvee_{i \leq k} z_2 = c_i) \wedge p_{\mathsf{AnnoEq}}(z_1, z_2, z_3)))$. 
        
	\textbf{Only if:} 
        $\mathcal{M}_{\sim}, \evalit \models p_{\mathsf{D}}(\overrightarrow{c})$ 
       holds by assumption. To prove that 
	$\mathcal{M}_{\sim}, \evalit \models    
			\forall z_1.\forall z_2.(p_{\mathsf{D}}(z_1,z_2) \wedge 
			(\bigvee_{i \leq k} z_1 = c_i) \wedge 
			\exists z_3.p_{\mathsf{AnnoEq}}(z_1,z_2,z_3) \supset 
			(\bigvee_{i \leq l} z_2 = c_i))$, 
	let $\evalit'$ be almost exactly $\evalit$ except that $\evalit'(z_1)$ and/or $\evalit'(z_2)$ 
	may be different from $\evalit(z_1)$ and/or $\evalit(z_2)$.  
       \begin{description} 
	       \item[\textbf{Case $\evalit'(z_1) \not\in ObjStmts$ OR $\evalit'(z_2) \not\in ObjStmts$}:]  
		       $\mathcal{M}_{\sim}, \evalit' \not\models p_{\mathsf{D}}(z_1, 
		       z_2)$. 
	       \item[\textbf{Case $\evalit'(z_1) = \evalit'(z_2)$}:] 
		       $\mathcal{M}_{\sim}, \evalit' \not\models p_{\mathsf{D}}(z_1, 
		       z_2)$. 
		\item[\textbf{Case $\evalit'(z_1), \evalit'(z_2) \in ObjStmts$ 
			AND $\evalit'(z_1) \not= \evalit'(z_2)$}:]{\ }  
			$\mathcal{M}, \evalit' \models p_{\mathsf{D}}(z_1, z_2)$ 
			holds. 
			\begin{description} 
				\item[\textbf{Sub-case $\evalit'(z_2) 
					\in ObjStmts'$}:] 
					$\mathcal{M}_{\sim}, \evalit' \models 
					(\bigvee_{i \leq l} z_2 = c_i)$. 
				\item[\textbf{Sub-case $\evalit'(z_2) \not\in 
					ObjStmts'$}:]   
					$\mathcal{M}_{\sim}, \evalit' \not\models 
					(\bigvee_{i \leq l} z_2 = c_i)$. 
					We prove 
					that $\mathcal{M}_{\sim}, \evalit' 
					\not\models 
			(\bigvee_{i \leq k} z_1 = c_i) \wedge 
			\exists z_3.p_{\mathsf{AnnoEq}}(z_1,z_2,z_3)$. 
    Suppose there is some $\evalit''$ which is almost exactly $\evalit'$ 
					except that $\evalit''(z_3)$ 
					is such that 
					$\mathcal{M}, \evalit'' \models 
					p_{\mathsf{AnnoEq}}(z_1, z_2, z_3)$.  
				Then, we must prove that 
					$\mathcal{M}, \evalit'' \not\models 
					\bigvee_{i \leq k} z_1 = c_i$. 
			  Suppose---by showing contradiction---
			   $\evalit''(z_1) \in \{\evalit''(c_1), \ldots, 
					\evalit''(c_k)\}$. 
					Then, $\evalit''(z_2) \in \text{cl}_{\sim}(\evalit''(c_1), \ldots, 
					\evalit''(c_k)\})$. 
				  However, by assumption,
					$\text{cl}_{\sim}(\{\evalit''(c_1), \ldots,\linebreak \evalit''(c_k)\}) = ObjStmts'$, contradiction. 
			\end{description} 
       \end{description} 
	Hence, $\mathcal{M}_{\sim}, \evalit \models    
			\forall z_1.\forall z_2.(p_{\mathsf{D}}(z_1,z_2) \wedge 
			(\bigvee_{i \leq k} z_1 = c_i) \wedge 
			\exists z_3.p_{\mathsf{AnnoEq}}(z_1,z_2,z_3) \supset 
			(\bigvee_{i \leq l} z_2 = c_i))$. 
	
	To prove that $\mathcal{M}_{\sim}, \evalit \models 
	\forall z_1.(p_{\mathsf{D}}(z_1) \wedge (\bigvee_{k < i \leq l} z_1 = c_i) \supset 
			\exists z_2.\exists z_3.((\bigvee_{i \leq k} z_2 = c_i) \wedge p_{\mathsf{AnnoEq}}(z_1, z_2, z_3)))$, let $\evalit'$ be almost exactly 
			$\evalit$ except that $\evalit'(z_1)$ and 
			$\evalit(z_1)$ may be different. 
    \begin{description} 
	    \item[\textbf{Case $\evalit'(z_1) \not\in ObjStmts$}:]  
		     $\mathcal{M}_{\sim}, \evalit' \not\models p_{\mathsf{D}}(z_1)$.
	     \item[\textbf{Case $\evalit'(z_1) \in \{\evalit'(c_1), \ldots, 
		     \evalit'(c_k)\}$}:]  
		      $\mathcal{M}_{\sim}, \evalit' \not\models \bigvee_{k < i 
		      \leq l} z_1 = c_i$. 
	      \item[\textbf{Case $\evalit'(z_1) \in ObjStmts' \backslash 
		      \{\evalit'(c_1), \ldots, \evalit'(c_k)\}$}:] 
		      We  must prove that 
		      $\mathcal{M}_{\sim}, \evalit' \models 
	\exists z_2.\exists z_3.((\bigvee_{i \leq k} z_2 = c_i) \wedge p_{\mathsf{AnnoEq}}(z_1, z_2, z_3))$. 
		    But by assumption, 
		    $ObjStmts' = \text{cl}_{\sim}(\{\evalit'(c_1), \ldots,
		    \evalit'(c_k)\})$. So, there is 
		    some node $u \in \{\evalit'(c_1), \ldots, \evalit'(c_k)\}$ 
		    such that 
		       $\evalit'(z_2) \in \text{cl}_{\sim}(\{u\})$.  
    \end{description} 
   Hence, $\mathcal{M}_{\sim}, \evalit \models 
	\forall z_1.(p_{\mathsf{D}}(z_1) \wedge (\bigvee_{k < i \leq l} z_1 = c_i) \supset 
			\exists z_2.\exists z_3.((\bigvee_{i \leq k} z_2 = c_i) \wedge p_{\mathsf{AnnoEq}}(z_1, z_2, z_3)))$.  \\ 

\textbf{If}: Let $\evalit'$ be such that 
$\mathcal{M}_{\sim}, \evalit' \models 
	kl\text{-}\mathsf{CL}(\overrightarrow{c})$. We have to 
	prove that, for every $\evalit$ which is almost exactly $\evalit'$ 
	except for the variable assignment, 
	$\{\evalit(c_1), \ldots, \evalit(c_l)\}$ ($0 < k \leq l$) 
	is a subset $ObjStmts'$ of $ObjStmts$ with  
	$ObjStmts' = \text{cl}_{\sim}(\{\evalit(c_1), \ldots, \evalit(c_k)\})$. 
	But $kl\text{-}\mathsf{CL}(\overrightarrow{c})$ 
	is a well-formed formula, so the difference in variable assignment makes no 
	difference in satisfiability judgement. Hence, it suffices to prove that 
	$\{\evalit'(c_1), \ldots, \evalit'(c_l)\}$ is $ObjStmts'$, 
	and that $ObjStmts' = \text{cl}_{\sim}(\{\evalit'(c_1), \ldots, \evalit'(c_k)\})$. 

	First, $\mathcal{M}_{\sim}, \evalit' \models p_{\mathsf{D}}(\overrightarrow{c})$, so 
	$\{\evalit'(c_1), \ldots, \evalit'(c_l)\}$ is indeed a subset 
	of $ObjStmts$ 
	with size $l$. 
	Second, 
	$\mathcal{M}_{\sim}, \evalit \models    
			\forall z_1.\forall z_2.(p_{\mathsf{D}}(z_1,z_2) \wedge 
			(\bigvee_{i \leq k} z_1 = c_i) \wedge 
			\exists z_3.p_{\mathsf{AnnoEq}}(z_1,z_2,z_3) \supset 
			(\bigvee_{i \leq l} z_2 = c_i))$, 
	so, for any $u \in \{\evalit'(c_1), \ldots, \evalit'(c_k)\}$ 
	and any $u' \in ObjStmts$, 
	if $u' \in \text{cl}_{\sim}(\{u\})$, then 
	$u' \in ObjStmts'$. 
	We also have to show that $ObjStmts'$ does not include  
	any $u$ in $ObjStmts \backslash \text{cl}_{\sim}(\{\evalit'(c_1),\linebreak
	\ldots, \evalit'(c_k)\})$. 
	But $\mathcal{M}_{\sim}, \evalit \models 
	\forall z_1.(p_{\mathsf{D}}(z_1) \wedge (\bigvee_{k < i \leq l} z_1 = c_i) \supset 
			\exists z_2.\exists z_3.(\linebreak(\bigvee_{i \leq k} z_2 = c_i) \wedge p_{\mathsf{AnnoEq}}(z_1, z_2, z_3)))$, so 
		for any $u \in ObjStmts' \backslash 
		\{\evalit'(c_1), \ldots,\linebreak \evalit'(c_k)\}$, 
		there is some $u' \in \{\evalit'(c_1), \ldots, \evalit'(c_k)\}$
		 such that $u \in \text{cl}_{\sim}(u')$.  

		 Consequently, $ObjStmts' = \text{cl}_{\sim}(\{\evalit'(c_1), 
		 \ldots, \evalit'(c_k)\})$, as required. \\

	 Suppose $0=k$. We must prove that $\{ \} = \text{cl}_{\sim}(\{\})$ iff$^*$ 
	 $\mathcal{M}_{\sim}, \evalit \models 00\text{-}\mathsf{CL}$ 
	 iff$^*$ $\mathcal{M}_{\sim}, \evalit \models \top$ (Definition \ref{def_cl}). 

\textbf{Only if}:  We have to prove 
that, for some $\evalit'$ which is almost exactly $\evalit$ except 
for variable assignment, 
$\mathcal{M}_{\sim}, \evalit' \models \top$. But $\mathcal{M}_{\sim}, \evalit'' \models \top$ 
holds for any $\evalit''$, so there is nothing to show. 

\textbf{If}: We have to prove that $\{ \} = \text{cl}_{\sim}(\{\})$. 
But this follows from the definition of cl$_{\sim}$. \hfill$\Box$ 

\end{proof} 

\begin{definition}[kN-WCF]\label{def_wcf} \rm 
	\noindent {\small $kN\text{-}\mathsf{WCF}(t_1, t_2, \ldots, t_k)$} is: 
	\begin{itemize} 
		\item {\small $\top$} if $0 = k$.  
		\item {\small $k\text{-}\mathsf{CF}(t_1, \ldots, t_k) \wedge 
			(\bigwedge_{k < i \leq N} \forall y_{k+1}.\ldots.\forall y_i.(
			ki\text{-}\mathsf{CL}(t_1, \ldots, t_k, y_{k+1}, \ldots, y_{i}) 
			\supset i\text{-}\mathsf{CF}(t_1,\linebreak \ldots, t_k, y_{k+1}, \ldots, y_i)))$ 
			if $1 \leq k$.} 
	\end{itemize} 
\end{definition} 
$kN\text{-}\mathsf{WCF}(t_1, t_2, \ldots, t_k)$ is for judging whether 
$k$ nodes are wide-conflict-free. The maximum number of graph nodes 
of an annotated graph takes the place of $N$. (Of course, at this stage, no annotated graph---the semantic information---is present, so this is just intuition for the following lemma.)  
\begin{lemma}[Characterisability of a wide-conflict-free set] \label{lem_kn_wcf} 
     Given {\small $\mathcal{M}_{\sim} \equiv 
	((ObjStmts,ObjE), Obj\Pi)$},   
	let $N$ be $|ObjStmts|$ and $k$ be $0 \leq k \leq N$. 
	Let $\evalit$ be such that: $\evalit(p_{\mathsf{A}})$ is 
	\resizebox{3cm}{!}{$[\star 1:\{\} \xrightarrow{\{attacks\}} \star 2:\{\}]$} if $p_{\mathsf{A}}$'s arity is 2 and   
\resizebox{3cm}{!}{$
[\!
\begin{tikzpicture}[baseline=(n.base)]
  \node (n) {$\star 1:\{\}$};
  \draw[->, looseness=8, out=10, in=-10]
    (n) to node[right, xshift=2pt] {\scriptsize$\{attacks\}$} (n);
\end{tikzpicture}
\!]
$}
	if it is 1; $\evalit(p_{\mathsf{D}})$ 
	is 
	\resizebox{3cm}{!}{$[\star 1:\{\} \ \star 2:\{\} \ \cdots \  \star n:\{\}] $} 
	with $n$ being the arity of $p_{\mathsf{D}}$; 
	and 
	$\evalit(p_{\mathsf{AnnoEq}})$ is 
	\resizebox{3cm}{!}{$[\star 1:\{\star 3\} \ \star 2:\{\star 3\}]$}.  
	Then $\{\evalit(c_1), \ldots,\evalit(c_k)\}$ is  
	wide-conflict-free 
	iff$^*$ $\mathcal{M}_{\sim}, \evalit 
	\models kN\text{-}\mathsf{WCF}(c_1, \ldots, c_k)$. 
\end{lemma} 
\begin{proof} 
Let $\overrightarrow{c}$ denote
	$c_1, \ldots, c_k$.  
	
	Suppose $0 < k$. 
	First, $\{\evalit(c_1), \ldots, \evalit(c_k)\}$ 
	is trivially not a wide-conflict-free set in $\mathcal{M}_{\sim}$ 
	when either of the following conditions holds. 
	\begin{itemize} 
		\item There is a pair of natural numbers $i$ and $j$ 
			such that $1 \leq i < j \leq k$ and that 
			$\evalit(c_i) = \evalit(c_j)$. In this case, 
			$\{\evalit(c_1), \ldots, \evalit(c_k)\}$ 
			is not a set. 
		\item There is a natural number $i$ such that 
			$1 \leq i \leq k$ and that $\evalit(c_i)$ 
			is not a member of $ObjStmts$.  
			$\{\evalit(c_1), \ldots, \evalit(c_k)\}$ 
			is not a subset of $ObjStmts$. 
	\end{itemize}  
	In these cases, $\mathcal{M}_{\sim}, \evalit \not\models  
	kN\text{-}\mathsf{WCF}(\overrightarrow{c})$. 
	
	Now, let $\{\evalit(c_1), \ldots, \evalit(c_k)\}$ ($0 < k$) be a subset 
	of $ObjStmts$, which we denote by $ObjStmts'$, with $|ObjStmts'| = k$. 
	We prove that $ObjStmts'$ is a wide conflict-free set  
	iff$^*$ $\mathcal{M}_{\sim}, \evalit \models    
	k\text{-}\mathsf{CF}(\overrightarrow{c}) \wedge (\bigwedge_{k < i \leq N} \forall y_{k+1}.\ldots.\forall y_i.(
			ki\text{-}\mathsf{CL}(\overrightarrow{c}, y_{k+1},\linebreak \ldots, y_{i}) 
			\supset i\text{-}\mathsf{CF}(\overrightarrow{c}, y_{k+1}, \ldots, y_i)))$. 
        
	\textbf{Only if:} Since $ObjStmts'$ is wide-conflict-free, it is also simple-conflict-free. 
	 By Lemma \ref{lem_k_cf}, it holds that 
	 $\mathcal{M}_{\sim}, \evalit \models k\text{-}\mathsf{CF}(\overrightarrow{c})$.  
	 By definition of cl$_{\sim}$, there is some necessarily unique $ObjStmts''\ (\subseteq ObjStmts)$ such that 
	 $ObjStmts'' = \text{cl}_{\sim}(ObjStmts')$. 
	 In case $ObjStmts'' = ObjStmts'$, 
	 $\bigwedge_{k < i \leq N} F = \top$ for any formula $F$, so there is nothing to prove further. 
	 Otherwise, we must prove that 
	$\mathcal{M}_{\sim}, \evalit \models \bigwedge_{k < i \leq N} \forall y_{k+1}.\ldots.\forall y_i.(
			ki\text{-}\mathsf{CL}(\overrightarrow{c}, y_{k+1}, \ldots, y_{i}) 
			\supset i\text{-}\mathsf{CF}(\overrightarrow{c}, y_{k+1}, \ldots,\linebreak y_i))$. 
	Let $\evalit'$ be almost exactly $\evalit$ except that, for any $k < i' \leq i$, $\evalit'(y_{i'})$ 
	may be different from $\evalit(y_{i'})$. 
   
	\begin{description} 
		\item[\textbf{Case some $\evalit'(y_{i'})$ is not in $ObjStmts$}:] 
			$\mathcal{M}_{\sim}, \evalit' \not\models ki\text{-}\mathsf{CL}(\overrightarrow{c}, 
			y_{k+1}, \ldots, y_i)$. 
		\item[\textbf{Case $\evalit'(y_{i'}) \in ObjStmts$ for every $k < i' \leq i$}:] {\ } 
				\begin{description} 
				      \item[\textbf{Sub-case $\evalit'(t_u) = \evalit'(t_v)$ for $t_u, t_v \in 
			\{c_1, \ldots, y_i\}$ with $t_u \not= t_v$}:] $\mathcal{M}_{\sim}, \evalit' \not\models ki\text{-}\mathsf{CL}(\overrightarrow{c}, 
			y_{k+1}, \ldots, y_i)$.   
		\item[\textbf{Sub-case, otherwise}:]  
			If $\{\evalit'(c_1), \ldots, \evalit'(y_i)\} \not= \text{cl}_{\sim}(\{\evalit'(c_1), \ldots, 
						\evalit'(c_k)\})$, then by Lemma \ref{lem_kl_cl} 
						it holds that 
						$\mathcal{M}_{\sim}, \evalit' \not\models ki\text{-}\mathsf{CL}(\overrightarrow{c}, y_{k+1}, \ldots, y_i)$. Otherwise, $\mathcal{M}_{\sim}, \evalit' \models ki\text{-}\mathsf{CL}(\overrightarrow{c}, 
						y_{k+1}, \ldots, y_i)$, but 
						by assumption, cl$_{\sim}(\{\evalit'(c_1), \ldots, \evalit'(c_k)\})$ is 
						simple-conflict-free, so 
						$\mathcal{M}_{\sim}, \evalit' \models i\text{-}\mathsf{CF}(\overrightarrow{c}, 
						y_{k+1}, \ldots, y_i)$. 
				\end{description}
	\end{description} 
	Hence, $\mathcal{M}_{\sim}, \evalit \models  \bigwedge_{k < i \leq N} \forall y_{k+1}.\ldots.\forall y_i.(
			ki\text{-}\mathsf{CL}(\overrightarrow{c}, y_{k+1}, \ldots, y_{i}) 
			\supset i\text{-}\mathsf{CF}(\overrightarrow{c}, y_{k+1},\linebreak \ldots, y_i))$. 
			Consequently, 
			$\mathcal{M}_{\sim}, \evalit \models  kN\text{-}\mathsf{WCF}(\overrightarrow{c})$, as required. 
			\\

\textbf{If}: Let $\evalit'$ be such that 
$\mathcal{M}_{\sim}, \evalit' \models kN\text{-}\mathsf{WCF}(\overrightarrow{c})$.  
	We have to 
	prove that, for every $\evalit$ which is almost exactly $\evalit'$ 
	except for the variable assignment, 
	$\{\evalit(c_1), \ldots, \evalit(c_k)\}$ ($0 < k$) 
	is a wide-conflict-free set. 
	But $kN\text{-}\mathsf{WCF}(\overrightarrow{c})$ 
	is a well-formed formula, so the difference in variable assignment makes no 
	difference in satisfiability judgement. Hence, it suffices to prove that 
	$\{\evalit'(c_1), \ldots, \evalit'(c_k)\}$ is a wide-conflict-free set. 

	First, $\mathcal{M}_{\sim}, \evalit' \models kN\text{-}\mathsf{WCF}(\overrightarrow{c})$ 
	and therefore $\mathcal{M}_{\sim}, \evalit' \models 
	 p_{\mathsf{D}}(\overrightarrow{c})$ holds. So, 
	 $\{\evalit'(c_1), \ldots, \evalit'(c_k)\}$ is a subset of $ObjStmts$ 
	 with size $k$. 

	 Second, $\mathcal{M}_{\sim}, \evalit' \models k\text{-}\mathsf{CF}(\overrightarrow{c})$ holds. 
	 Suppose $k = N$, then, by Lemma \ref{lem_k_cf}, $\{\evalit'(c_1), \ldots, \evalit'(c_k)\}$ is 
	 simple-conflict-free. There is no more arguments in $ObjStmts$, so 
	 it is 
	  trivially wide-conflict-free. 
	 
	 Suppose $k < N$. We have to show that, for any $ObjStmts''$ with $ObjStmts' \subset 
	 ObjStmts'' \subseteq ObjStmts$, if $ObjStmts'' = \text{cl}_{\sim}(ObjStmts')$, 
	 then $ObjStmts''$ is simple-conflict-free. 
	 However, by assumption, $\mathcal{M}_{\sim}, \evalit' \models 
	 \bigwedge_{k < i \leq N} \forall y_{k+1}.\ldots y_i.(\linebreak ki\text{-}\mathsf{CL}(\overrightarrow{c},
	 y_{k+1}, \ldots, y_i) \supset i\text{-}\mathsf{CF}(\overrightarrow{c}, y_{k+1}, \ldots, y_i))$. 
	 Hence, for any $\evalit''$ which is almost exactly $\evalit'$ except that, 
	 for any $k+1 \leq i' \leq i$, $\evalit''(i')$ may not be $\evalit'(i')$, 
	 if $\{\evalit''(c_1), \ldots, \evalit''(c_k), \evalit''(y_{k+1}), \ldots, \evalit''(y_{i'})\} 
	 = \text{cl}_{\sim}(ObjStmts')$, then it is simple-conflict-free, as required. \\ 

	 Similarly if $k= 0$.  \hfill$\Box$ 
\end{proof} 

\begin{definition}[k-DF]\label{def_df} \rm 
	\noindent {\small  $k\text{-}\mathsf{DF}(t, t_1, t_2, \ldots, t_k)$} is: 
	     \begin{itemize}  
		     \item  {\small $p_{\mathsf{D}}(t) \wedge 
			     \forall y.(p_{\mathsf{D}}(y) \supset 
			     (\neg y = t \wedge \neg p_{\mathsf{A}}(y,t)) \vee (y = t \wedge 
			     \neg p_{\mathsf{A}}(y)))$} if $0 = k$. 
		     \item {\small $p_{\mathsf{D}}(t) \wedge 
			     p_{\mathsf{D}}(t_1, t_2, \ldots, t_k) 
			     \wedge \forall y.\exists x'.(p_{\mathsf{D}}(y) \wedge   
			     ((\neg y = t \wedge p_{\mathsf{A}}(y,t)) \vee  
			     (y = t \wedge p_{\mathsf{A}}(y)))
			     \supset p_{\mathsf{D}}(x') \wedge (\bigvee_{i\leq k} x' = t_i)
			     \wedge ((\neg x' = y \wedge p_{\mathsf{A}}(x', y)) \vee (x' = y \wedge 
			     p_{\mathsf{A}}(x'))))$} if $1 \leq k$.  \hfill$\spadesuit$ 
	     \end{itemize} 
\end{definition} 
$k\text{-}\mathsf{DF}(t, t_1, t_2, \ldots, t_k)$ is for judging whether $k$ nodes simple-defend a node.  
\begin{lemma}[Characterisability of simple-defence of a node]\label{lem_k_df} 
	Given {\small $\mathcal{M}_{\sim} \equiv 
	((ObjStmts,ObjE), Obj\Pi)$}, 
 	let $\evalit$ be such that: $\evalit(p_{\mathsf{A}})$ is 
	\resizebox{3cm}{!}{$[\star 1:\{\} \xrightarrow{\{attacks\}} \star 2:\{\}]$} if $p_{\mathsf{A}}$'s arity is 2 and   
\resizebox{3cm}{!}{$
[\! 
\begin{tikzpicture}[baseline=(n.base)]
  \node (n) {$\star 1:\{\}$};
  \draw[->, looseness=8, out=10, in=-10]
    (n) to node[right, xshift=2pt] {\scriptsize$\{attacks\}$} (n);
\end{tikzpicture}
\!]
$}
	if it is 1; and $\evalit(p_{\mathsf{D}})$ 
	is 
	\resizebox{3cm}{!}{$[\star 1:\{\} \ \star 2:\{\} \ \cdots \  \star n:\{\}] $} 
	with $n$ being the arity of $p_{\mathsf{D}}$. 
	Let $k$ be $0 \leq k \leq |ObjStmts|$. 
	Then $\{\evalit(c_1), \ldots,\linebreak \evalit(c_k)\}$ 
	simple-defends a node $\evalit(c)$ 
	iff$^*$ $\mathcal{M}_{\sim}, \evalit 
	\models k\text{-}\mathsf{DF}(c, c_1, \ldots, c_k)$. 
\end{lemma}  
\begin{proof} 
   Let $\overrightarrow{c}$ denote $c_1, \ldots, c_k$.  

	Suppose $0 < k$. First, $\{\evalit(c_1), \ldots, \evalit(c_k)\}$ is 
	not a set of nodes when either of the following conditions holds. 

	\begin{itemize} 
		\item There is a pair of natural numbers $i$ and $j$ such that $1 \leq i < j \leq k$ 
			and that $\evalit(c_i) = \evalit(c_j)$. In this case, 
			$\{\evalit(c_1), \ldots, \evalit(c_k)\}$ is not a set. 
		\item There is a natural number $i$ such that $1 \leq i \leq k$ and that 
			$\evalit(c_i)$ is not a member of $ObjStmts$. $\{\evalit(c_1), \ldots, \evalit(c_k)\}$ 
			is not a subset of $ObjStmts$.  
	\end{itemize} 
	In these cases, $\mathcal{M}_{\sim}, \evalit \not\models k\text{-}\mathsf{DF}(c, \overrightarrow{c})$. 
	Also, if $\evalit(c) \not\in ObjStmts$, $\mathcal{M}_{\sim}, \evalit \not\models k\text{-}\mathsf{DF}(c, 
	\overrightarrow{c})$.  

	Now, let $\evalit(c)$ be a member of $ObjStmts$. 
	Let $\{\evalit(c_1), \ldots, \evalit(c_k)\}$ ($0 < k$) be a subset of $ObjStmts$, 
	which we denote by $ObjStmts'$, with $|ObjStmts'| = k$. We prove that 
	$ObjStmts'$ simple-defends $\evalit(c)$ iff$^*$ $\mathcal{M}_{\sim}, \evalit \models 
	p_{\mathsf{D}}(c) \wedge 
			     p_{\mathsf{D}}(\overrightarrow{c}) 
			     \wedge \forall y.\exists x'.(p_{\mathsf{D}}(y) \wedge (  
			     (\neg y = c \wedge p_{\mathsf{A}}(y,c)) \vee  
			     (y = c \wedge p_{\mathsf{A}}(y)))
			     \supset p_{\mathsf{D}}(x') \wedge (\bigvee_{i\leq k} x' = c_i)
			     \wedge ((\neg x' = y \wedge p_{\mathsf{A}}(x', y)) \vee (x' = y \wedge 
			     p_{\mathsf{A}}(x'))))$.  

			     \textbf{Only if}: 
	By assumption,  $\mathcal{M}_{\sim}, \evalit \models 
	p_{\mathsf{D}}(c) \wedge 
			     p_{\mathsf{D}}(\overrightarrow{c})$ holds.  
		We must prove that $\mathcal{M}_{\sim}, \evalit \models 
			     \forall y.\exists x'.(p_{\mathsf{D}}(y) \wedge (  
			     (\neg y = c \wedge p_{\mathsf{A}}(y,c)) \vee  
			     (y = c \wedge p_{\mathsf{A}}(y)))
			     \supset p_{\mathsf{D}}(x') \wedge (\bigvee_{i\leq k} x' = c_i)
			     \wedge ((\neg x' = y \wedge p_{\mathsf{A}}(x', y)) \vee (x' = y \wedge 
			     p_{\mathsf{A}}(x'))))$.  
	       Let $\evalit'$ be almost exactly $\evalit$ except that 
	       $\evalit'(y)$ may not be $\evalit(y)$. 
	       \begin{description} 
		       \item[\textbf{Case $\evalit'(y) \not\in ObjStmts$}:]  
			       $\mathcal{M}_{\sim}, \evalit' \not\models   
			       p_{\mathsf{D}}(y) \wedge (  
			     (\neg y = c \wedge p_{\mathsf{A}}(y,c)) \vee  
			       (y = c \wedge p_{\mathsf{A}}(y)))$.  
		       \item[\textbf{Case $\evalit'(y) \in ObjStmts$}:]{\ } 
			       \begin{description}  
				       \item[\textbf{Sub-case no $u \in ObjStmts'$ is such that $(\evalit'(y), c) \in ObjE$}:] 
					       $\mathcal{M}_{\sim}, \evalit' \not\models   
			       p_{\mathsf{D}}(y) \wedge (  
			     (\neg y = c \wedge p_{\mathsf{A}}(y,c)) \vee  
			       (y = c \wedge p_{\mathsf{A}}(y)))$.  
				       \item[\textbf{Sub-case, otherwise}:]   
						$\mathcal{M}_{\sim}, \evalit' \models   
			       p_{\mathsf{D}}(y) \wedge (  
			     (\neg y = c \wedge p_{\mathsf{A}}(y,c)) \vee  
			       (y = c \wedge p_{\mathsf{A}}(y)))$. We must prove 
					       that $\mathcal{M}_{\sim}, \evalit' \models \exists x'.(
					       p_{\mathsf{D}}(x') \wedge (\bigvee_{i\leq k} x' = c_i)
			     \wedge ((\neg x' = y \wedge p_{\mathsf{A}}(x', y)) \vee (x' = y \wedge 
					       p_{\mathsf{A}}(x'))))$. 
					       By assumption, $ObjStmts'$ simple-defends $\evalit'(c)$. 
					       Therefore, there is some $u \in ObjStmts'$ such that 
				               $(u, \evalit'(y)) \in ObjE$. 
					       Let $\evalit''$ be almost exactly $\evalit'$ except that 
					       $\evalit''(x')$ is the member $u$ of $ObjStmts'$. 
					       Then, 
					       $\mathcal{M}_{\sim}, \evalit'' \models 
						p_{\mathsf{D}}(x') \wedge 
						(\bigvee_{i \leq k} x' = c_i) \wedge 
						((\neg x' = y \wedge p_{\mathsf{A}}(x',y)) \vee 
						(x' = y \wedge p_{\mathsf{A}}x'))$. 
			       \end{description} 
	       \end{description} 
	       Consequently, $\mathcal{M}_{\sim}, \evalit \models \forall y.\exists x'.(p_{\mathsf{D}}(y) \wedge (  
			     (\neg y = c \wedge p_{\mathsf{A}}(y,c)) \vee  
			     (y = c \wedge p_{\mathsf{A}}(y)))
			     \supset p_{\mathsf{D}}(x') \wedge (\bigvee_{i\leq k} x' = c_i)
			     \wedge ((\neg x' = y \wedge p_{\mathsf{A}}(x', y)) \vee (x' = y \wedge 
			     p_{\mathsf{A}}(x'))))$, as required. \\ 

	       \textbf{If}: Let $\evalit'$ be such that 
		 $\mathcal{M}_{\sim}, \evalit' \models k\text{-}\mathsf{DF}(c, \overrightarrow{c})$.  
		We have to 
		prove that, for every $\evalit$ which is almost exactly $\evalit'$ 
		except for the variable assignment, 
		$\{\evalit(c_1), \ldots, \evalit(c_k)\}$ ($0 < k$) 
		simple-defends $\evalit(c)$. 
		But $k\text{-}\mathsf{DF}(c, \overrightarrow{c})$ 
		is a well-formed formula, so the difference in variable assignment makes no 
		difference in satisfiability judgement. Hence, it suffices to prove that 
		$\{\evalit'(c_1), \ldots, \evalit'(c_k)\}$ simple-defends $\evalit'(c)$. 
		
		First, $\mathcal{M}_{\sim}, \evalit' \models p_{\mathsf{D}}(c) \wedge p_{\mathsf{D}}(\overrightarrow{c})$. 
		So, $\{\evalit'(c_1), \ldots, \evalit'(c_k)\}$ is a subset $ObjStmts'$ of $ObjStmts$, 
		and $\evalit'(c)$ is a member of $ObjStmts$. 

		Second, $\mathcal{M}_{\sim}, \evalit' \models \forall y.\exists x'.(p_{\mathsf{D}}(y) \wedge (  
			     (\neg y = c \wedge p_{\mathsf{A}}(y,c)) \vee  
			     (y = c \wedge p_{\mathsf{A}}(y)))
			     \supset p_{\mathsf{D}}(x') \wedge (\bigvee_{i\leq k} x' = c_i)
			     \wedge ((\neg x' = y \wedge p_{\mathsf{A}}(x', y)) \vee (x' = y \wedge 
			     p_{\mathsf{A}}(x'))))$. Hence, for any $u \in ObjStmts$, 
			     if $(u, \evalit'(c)) \in ObjE$, then 
			     there is some $u' \in ObjStmts'$ such that 
			     $(u', u) \in ObjE$, as required. \\ 

	        Suppose $0 = k$. We prove that  
		$\{ \}$ simple-defends a node $\evalit(c)$ iff$^*$ 
		$\mathcal{M}_{\sim}, \evalit \models 
		p_{\mathsf{D}}(c) \wedge 
			     \forall y.(p_{\mathsf{D}}(y) \supset 
			     (\neg y = c \wedge \neg p_{\mathsf{A}}(y,c)) \vee (y = c \wedge 
			     \neg p_{\mathsf{A}}(y)))$. 

	       \textbf{Only if}: By assumption, $\mathcal{M}_{\sim}, \evalit \models p_{\mathsf{D}}(c)$. 
	       We must prove that $\mathcal{M}_{\sim}, \evalit \models \forall y.(p_{\mathsf{D}}(y) \supset 
	       (\neg y = c \wedge \neg p_{\mathsf{A}}(y, c)) \vee (y = c \wedge \neg p_{\mathsf{A}}(y)))$.
	       Let $\evalit'$ be almost exactly $\evalit$ except that $\evalit'(y)$ may not be $\evalit(y)$. 
	       \begin{description} 
		       \item[\textbf{Case $\evalit'(y) \not\in ObjStmts$}:] 
			       $\mathcal{M}_{\sim}, \evalit' \not\models p_{\mathsf{D}}(y)$.  
		       \item[\textbf{Case $\evalit'(y) \in ObjStmts$}:]   
				$\mathcal{M}_{\sim}, \evalit' \models p_{\mathsf{D}}(y)$. 
				We must prove that $\mathcal{M}_{\sim}, \evalit' \models 
				(\neg y = c \wedge \neg p_{\mathsf{A}}(y, c)) \vee (y = c \wedge \neg p_{\mathsf{A}}(y))$. 
				However, by assumption, there is no $u \in ObjStmts$ such that 
				$(u, \evalit(c)) \in ObjE$. 
	       \end{description} 
	       Consequently, $\mathcal{M}_{\sim}, \evalit \models \forall y.(p_{\mathsf{D}}(y) \supset 
	       (\neg y = c \wedge \neg p_{\mathsf{A}}(y, c)) \vee (y = c \wedge \neg p_{\mathsf{A}}(y)))$ holds, as required.\\

	       \textbf{If}: 
	       Let $\evalit'$ be such that 
		 $\mathcal{M}_{\sim}, \evalit' \models 0\text{-}\mathsf{DF}(c)$.  
	       We must prove that $\{ \}$ simple-defends $\evalit(c)$ for any $\evalit$ which 
	       is almost exactly $\evalit'$ except for variable assignment. But $0\text{-}\mathsf{DF}(c)$ is well-formed, 
	       so the difference in variable assignment makes no difference in satisfiability judgement. Hence, 
	       it suffices to prove that $\{ \}$ simple-defends $\evalit'(c)$. 
	       
	       First, $\mathcal{M}_{\sim}, \evalit' \models p_{\mathsf{D}}(c)$, so $\evalit'(c)$ is a member of $ObjStmts$. 
	       Second, $\mathcal{M}_{\sim}, \evalit' \models \forall y.(p_{\mathsf{D}}(y) \supset 
	       (\neg y = c \wedge \neg p_{\mathsf{A}}(y, c)) \vee (y = c \wedge \neg p_{\mathsf{A}}(y)))$, 
	       so if $u \in ObjStmts$, it cannot be that $(u, \evalit'(c)) \in ObjE$, as required. \hfill$\Box$  
\end{proof} 
\begin{definition}[kN-WDF]\label{def_wdf} \rm 
	\noindent {\small  $kN\text{-}\mathsf{WDF}(t, t_1, t_2, \ldots, t_k)$} is: 
	     \begin{itemize}  
		     \item  {\small $p_{\mathsf{D}}(t) \wedge \bigwedge_{1 \leq i \leq N} 
			     \forall y_{2}.\ldots.\forall y_i.(1i\text{-}\mathsf{CL}(t, y_2, \ldots, y_i) 
			     \supset \forall y'.(p_{\mathsf{D}}(y') \wedge (y' = t 
			     \vee (\bigvee_{2 \leq i' \leq i} y' = y_{i'})) \supset 
			     k\text{-}\mathsf{DF}(y', t_1, \ldots, t_k)))$} 
			     if $0 = k$. 
		     \item {\small $p_{\mathsf{D}}(t) \wedge p_{\mathsf{D}}(t_1, t_2, \ldots, t_k) 
			     \wedge 
				\bigwedge_{1 \leq i \leq N} 
			     \forall y_{2}.\ldots.\forall y_i.(1i\text{-}\mathsf{CL}(t, y_2, \ldots, y_i) 
			     \supset \forall y'.(p_{\mathsf{D}}(y') \wedge (y' = t 
			     \vee (\bigvee_{2 \leq i' \leq i} y' = y_{i'})) 
			     \supset k\text{-}\mathsf{DF}(y', 
			     t_1, \ldots, t_k)))$}
			     if $1 \leq k$.  \hfill$\spadesuit$ 
	     \end{itemize} 
\end{definition} 
$kN\text{-}\mathsf{WDF}(t, t_1, t_2, \ldots, t_k)$ is for whether $k$ nodes 
wide-defend a node.  

\begin{lemma}[Characterisability of wide-defence of a node] \label{lem_kn_wdf} 
   Given {\small $\mathcal{M}_{\sim} \equiv 
	((ObjStmts,ObjE), Obj\Pi)$},   
	let $N$ be $|ObjStmts|$ and $k$ be $0 \leq k \leq N$. 
	Let $\evalit$ be such that: $\evalit(p_{\mathsf{A}})$ is 
	\resizebox{3cm}{!}{$[\star 1:\{\} \xrightarrow{\{attacks\}} \star 2:\{\}]$} if $p_{\mathsf{A}}$'s arity is 2 and   
\resizebox{3cm}{!}{$
[\!
\begin{tikzpicture}[baseline=(n.base)]
  \node (n) {$\star 1:\{\}$};
  \draw[->, looseness=8, out=10, in=-10]
    (n) to node[right, xshift=2pt] {\scriptsize$\{attacks\}$} (n);
\end{tikzpicture}
\!]
$}
	if it is 1; $\evalit(p_{\mathsf{D}})$ 
	is 
	\resizebox{3cm}{!}{$[\star 1:\{\} \ \star 2:\{\} \ \cdots \  \star n:\{\}] $} 
	with $n$ being the arity of $p_{\mathsf{D}}$; 
	and 
	$\evalit(p_{\mathsf{AnnoEq}})$ is 
	\resizebox{3cm}{!}{$[\star 1:\{\star 3\} \ \star 2:\{\star 3\}]$}.  
	Then $\{\evalit(c_1), \ldots,\evalit(c_k)\}$ wide-defends 
	$\evalit(c)$ 
	iff$^*$ $\mathcal{M}_{\sim}, \evalit 
	\models kN\text{-}\mathsf{WDF}(c, c_1, \ldots, c_k)$. 
\end{lemma} 
\begin{proof} 
   Let  $\overrightarrow{c}$ denote $c_1, \ldots, c_k$.  

	Suppose $0 < k$. First, $\{\evalit(c_1), \ldots, \evalit(c_k)\}$ is 
	not a set of nodes when either of the following conditions holds. 

	\begin{itemize} 
		\item There is a pair of natural numbers $i$ and $j$ such that $1 \leq i < j \leq k$ 
			and that $\evalit(c_i) = \evalit(c_j)$. In this case, 
			$\{\evalit(c_1), \ldots, \evalit(c_k)\}$ is not a set. 
		\item There is a natural number $i$ such that $1 \leq i \leq k$ and that 
			$\evalit(c_i)$ is not a member of $ObjStmts$. $\{\evalit(c_1), \ldots, \evalit(c_k)\}$ 
			is not a subset of $ObjStmts$.  
	\end{itemize} 
	In these cases, $\mathcal{M}_{\sim}, \evalit \not\models kN\text{-}\mathsf{WDF}(c, \overrightarrow{c})$. 
	Also, if $\evalit(c) \not\in ObjStmts$, $\mathcal{M}_{\sim},\linebreak \evalit \not\models kN\text{-}\mathsf{WDF}(c, 
	\overrightarrow{c})$.  

	Now, let $\evalit(c)$ be a member of $ObjStmts$. 
	Let $\{\evalit(c_1), \ldots, \evalit(c_k)\}$ ($0 < k$) be a subset of $ObjStmts$, 
	which we denote by $ObjStmts'$, with $|ObjStmts'| = k$. We prove that 
	$ObjStmts'$ wide-defends $\evalit(c)$ iff$^*$ $\mathcal{M}_{\sim}, \evalit \models 
	p_{\mathsf{D}}(c) \wedge p_{\mathsf{D}}(\overrightarrow{c}) \wedge \bigwedge_{1 \leq i \leq N} 
			     \forall y_{2}.\ldots.\forall y_i.(1i\text{-}\mathsf{CL}(c, y_2, \ldots, y_i) 
			     \supset \forall y'.(p_{\mathsf{D}}(y') \wedge (y' = c 
			     \vee (\bigvee_{2 \leq i' \leq i} y' = y_{i'})) \supset 
			     k\text{-}\mathsf{DF}(y', \overrightarrow{c})))$. 

	\textbf{Only if}: By assumption,  $\mathcal{M}_{\sim}, \evalit \models 
			p_{\mathsf{D}}(c) \wedge p_{\mathsf{D}}(\overrightarrow{c})$ holds.   
			We must prove that, for each $1 \leq i \leq N$, 
			$\mathcal{M}_{\sim}, \evalit \models 
			\forall y_{2}.\ldots.\forall y_i.(1i\text{-}\mathsf{CL}(c, y_2, \ldots, y_i) 
			     \supset \forall y'.(\linebreak p_{\mathsf{D}}(y') \wedge (y' = c 
			     \vee (\bigvee_{2 \leq i' \leq i} y' = y_{i'})) \supset 
			     k\text{-}\mathsf{DF}(y', \overrightarrow{c})))$.

			     Suppose $1 < i \leq N$. Let $\evalit'$ be almost exactly $\evalit$ except that, 
			     for any $2 \leq i' \leq i$, $\evalit'(y_{i'})$ may be different from 
			     $\evalit(y_{i'})$. 

			     \begin{description} 
				     \item[\textbf{Case some $i'$ is such that $\evalit'(y_{i'}) \not\in ObjStmts$}:]  
					     By Lemma \ref{lem_kl_cl}, $\mathcal{M}_{\sim}, \evalit' \not\models 
			             1i\text{-}\mathsf{CL}(c, y_2, \ldots, y_i)$.   
			     \item[\textbf{Case  $\{\evalit'(c), \evalit'(y_2), \ldots, \evalit'(y_{i})\}$ is not a set}:] 
				     	 By Lemma \ref{lem_kl_cl}, $\mathcal{M}_{\sim}, \evalit' \not\models 
			             1i\text{-}\mathsf{CL}(c, y_2, \ldots, y_i)$.   
			     \item[\textbf{Case \text{cl}$_{\sim}(\{\evalit'(c)\}) \not= 
						     \{\evalit'(c), \evalit'(y_2), \ldots, \evalit'(y_{i})\}$}:]     
						     By Lemma \ref{lem_kl_cl}, $\mathcal{M}_{\sim}, \evalit' \not\models 
			             1i\text{-}\mathsf{CL}(c, y_2, \ldots, y_i)$.   
			     \item[\textbf{Case, otherwise}:]  
				     We prove that $\mathcal{M}_{\sim}, \evalit' \models 
						     \forall y'.(p_{\mathsf{D}}(y') \wedge (y' = c 
			     \vee (\bigvee_{2 \leq i' \leq i} y' = y_{i'})) \supset 
			     k\text{-}\mathsf{DF}(y', \overrightarrow{c}))$.  
						   Let $\evalit''$ be almost exactly $\evalit'$ except that 
						     $\evalit''(y')$ may be different from $\evalit'(y')$. 
						     \begin{description} 
							     \item[\textbf{Sub-case $\evalit''(y')\not\in ObjStmts$}:] 
								     $\mathcal{M}_{\sim}, \evalit'' \not\models  
								     p_{\mathsf{D}}(y') \wedge (y' = c 
								     \vee (\bigvee_{2 \leq i' \leq i} y' = y_{i'}))$.  
							     \item[\textbf{Sub-case $\evalit''(y') \not\in 
								     \{\evalit''(c), \evalit''(y_2), \ldots,
								     \evalit''(y_i)\}$}:] 
								     $\mathcal{M}_{\sim}, \evalit'' \not\models  
								     p_{\mathsf{D}}(y') \wedge (y' = c 
								     \vee (\bigvee_{2 \leq i' \leq i} y' = y_{i'}))$.   
							     \item[\textbf{Sub-case, otherwise}:] 
								     By assumption, for any $2 \leq i' \leq i$, 
								     $ObjStmts'$ simple-defends  
								     $\evalit''(y_{i'})$. It also simple-defends $\evalit''(c)$.
								     Hence, 
								     by Lemma \ref{lem_k_df}, 
								     $\mathcal{M}_{\sim}, \evalit'' \models 
								     k\text{-}\mathsf{DF}(y',\linebreak \overrightarrow{c})$. 
				     \end{description} 
			     \end{description} 
				Consequently, for each $1 < i \leq N$, 
			$\mathcal{M}_{\sim}, \evalit \models 
			\forall y_{2}.\ldots.\forall y_i.(1i\text{-}\mathsf{CL}(c, y_2, \ldots, y_i) 
			     \supset \forall y'.(p_{\mathsf{D}}(y') \wedge (y' = c 
			     \vee (\bigvee_{2 \leq i' \leq i} y' = y_{i'})) \supset 
			     k\text{-}\mathsf{DF}(y', \overrightarrow{c})))$, as required. 

				If $1 = i$, we must prove that 
			     $\mathcal{M}_{\sim}, \evalit \models 
			11\text{-}\mathsf{CL}(c) \supset \forall y'.(p_{\mathsf{D}}(y') \wedge (y' = c 
			     \vee \bot) \supset 
			     k\text{-}\mathsf{DF}(y', \overrightarrow{c}))$. The proof is similar. 
			     \\

			     \textbf{If}: 
			     Let $\evalit'$ be such that 
		 $\mathcal{M}_{\sim}, \evalit' \models kN\text{-}\mathsf{WDF}(c, \overrightarrow{c})$.  
		We have to 
		prove that, for every $\evalit$ which is almost exactly $\evalit'$ 
		except for the variable assignment, 
		$\{\evalit(c_1), \ldots, \evalit(c_k)\}$ ($0 < k$) 
		wide-defends $\evalit(c)$. 
		But $kN\text{-}\mathsf{WDF}(c, \overrightarrow{c})$ 
		is a well-formed formula, so the difference in variable assignment makes no 
		difference in satisfiability judgement. Hence, it suffices to prove that 
		$\{\evalit'(c_1), \ldots, \evalit'(c_k)\}$ wide-defends $\evalit'(c)$. 
		
		First, $\mathcal{M}_{\sim}, \evalit' \models p_{\mathsf{D}}(c) \wedge p_{\mathsf{D}}(\overrightarrow{c})$. 
		So, $\{\evalit'(c_1), \ldots, \evalit'(c_k)\}$ is a subset $ObjStmts'$ of $ObjStmts$, 
		and $\evalit'(c)$ is a member of $ObjStmts$. 

		Second, for each $1 \leq i \leq N$,  
		 $\mathcal{M}_{\sim}, \evalit' \models \forall y_{2}.\ldots.\forall y_i.(1i\text{-}\mathsf{CL}(c, y_2, \ldots, y_i) 
			     \supset \forall y'.(p_{\mathsf{D}}(y') \wedge (y' = c 
			     \vee (\bigvee_{2 \leq i' \leq i} y' = y_{i'})) \supset 
			     k\text{-}\mathsf{DF}(y', \overrightarrow{c})))$. 
			     Suppose $1 < i$. If there are $(i-1)$ members $u_2, \ldots, u_i$ of $ObjStmts$ 
			     such that cl$_{\sim}(\{\evalit'(c)\}) = \{\evalit'(c), u_2, \ldots, u_i\}$, 
			     then $ObjStmts'$ simple-defends each node in this set. 
			     Suppose $1 = i$. If cl$_{\sim}(\{\evalit'(c)\}) = \{\evalit'(c)\}$, then 
			     $ObjStmts'$ simple-defends $\evalit'(c)$, as required. \\

			     Suppose $0 = k$. We prove that  
		$\{ \}$ wide-defends a node $\evalit(c)$ iff$^*$ 
		$\mathcal{M}_{\sim}, \evalit \models 
		p_{\mathsf{D}}(c) \wedge 
			     \forall y.(p_{\mathsf{D}}(y) \supset 
			     (\neg y = c \wedge \neg p_{\mathsf{A}}(y,c)) \vee (y = c \wedge 
			     \neg p_{\mathsf{A}}(y)))$. This formula is different from the formula 
			     for $0 < k$ only by absence of $p_{\mathsf{D}}(\overrightarrow{c})$. Proof is similar. 
	        \hfill$\Box$  
\end{proof}

We use these formulas to write down several other formulas. 

\begin{definition}[k(N)-(W)ADM]\label{def_adm} \rm 
	\noindent {\small  $k\text{-}\mathsf{ADM}(t_1, \ldots, t_k)$ is:\\ $
	      k\text{-}\mathsf{CF}(t_1, \ldots, t_k) \wedge 
	      \bigwedge_{1 \leq j \leq k} k\text{-}\mathsf{DF}(t_j, t_1, \ldots, t_k)$}.  \\\\  
	      $kN\text{-}\mathsf{WADM}(t_1, \ldots, t_k)$ is: 
	      {\small $kN\text{-}\mathsf{WCF}(t_1, \ldots, t_k) \wedge 
	      \bigwedge_{1 \leq j \leq k} kN\text{-}\mathsf{WDF}(t_j, t_1, \ldots, t_k)$}.  
	      \hfill$\spadesuit$ 
\end{definition}    

\begin{proposition}[Characterisability of a simple-/wide-admissible set] \label{prop_admissible} 
   Given {\small $\mathcal{M}_{\sim} \equiv 
	((ObjStmts,ObjE), Obj\Pi)$},   
	let $N$ be $|ObjStmts|$ and $k$ be $0 \leq k \leq N$. 
	Let $\evalit$ be such that: $\evalit(p_{\mathsf{A}})$ is 
	\resizebox{3cm}{!}{$[\star 1:\{\} \xrightarrow{\{attacks\}} \star 2:\{\}]$} if $p_{\mathsf{A}}$'s arity is 2 and   
\resizebox{3cm}{!}{$
[\!
\begin{tikzpicture}[baseline=(n.base)]
  \node (n) {$\star 1:\{\}$};
  \draw[->, looseness=8, out=10, in=-10]
    (n) to node[right, xshift=2pt] {\scriptsize$\{attacks\}$} (n);
\end{tikzpicture}
\!]
$}
	if it is 1; $\evalit(p_{\mathsf{D}})$ 
	is 
	\resizebox{3cm}{!}{$[\star 1:\{\} \ \star 2:\{\} \ \cdots \  \star n:\{\}] $} 
	with $n$ being the arity of $p_{\mathsf{D}}$; 
	and 
	$\evalit(p_{\mathsf{AnnoEq}})$ is 
	\resizebox{3cm}{!}{$[\star 1:\{\star 3\} \ \star 2:\{\star 3\}]$}.  
	Then $\{\evalit(c_1), \ldots,\evalit(c_k)\}$ is 
	\begin{itemize} 
		\item a simple-admissible set iff$^*$ $\mathcal{M}_{\sim}, \evalit \models 
			k\text{-}\mathsf{CF}(c_1, \ldots, c_k) \wedge 
	      \bigwedge_{1 \leq j \leq k} k\text{-}\mathsf{DF}(c_j, c_1,\linebreak \ldots, c_k)$. 
      \item a wide-admissible set iff$^*$ $\mathcal{M}_{\sim}, \evalit \models kN\text{-}\mathsf{WCF}(c_1, \ldots, c_k) \wedge 
	      \bigwedge_{1 \leq j \leq k} kN\text{-}\mathsf{WDF}(\linebreak c_j, c_1, \ldots, c_k)$. 
	\end{itemize} 
\end{proposition} 
\begin{proof} 
	Follows immediately from Definition \ref{def_simple_wide_admissible} 
	and Lemmas \ref{lem_k_cf} through \ref{lem_kn_wdf}. \hfill$\Box$ 
\end{proof}

\begin{definition}[k(N)-(W-)(D/E-)CMP]\label{def_cmp}\rm {\ }\\ 
	\noindent {\small  $k\text{-}\mathsf{D}\text{-}\mathsf{CMP}(t_1, \ldots, t_k)$ is: $  
	      k\text{-}\mathsf{ADM}(t_1, \ldots, t_k)
	      \wedge \forall x.(k\text{-}\mathsf{DF}(x, t_1, \ldots, t_k) \supset 
		       \bigvee_{l \leq k}x = t_l)$}.\\\\ 
	{\small  $kN\text{-}\mathsf{W}\text{-}\mathsf{D}\text{-}\mathsf{CMP}(t_1, \ldots, t_k)$ is: $  
	      kN\text{-}\mathsf{WADM}(t_1, \ldots, t_k)
	      \wedge \forall x.(kN\text{-}\mathsf{WDF}(x, t_1, \ldots, t_k) \supset\\ 
		       \bigvee_{l \leq k}x = t_l)$}.\\ 
			      
	\noindent	${\small  k\mathsf{\text{-}E\text{-}CMP}(t_1, \ldots, t_k)}$ is: 
		  {\small $k\text{-}\mathsf{ADM}(t_1, \ldots, t_k) \wedge   
		      kk\text{-}\mathsf{CL}(t_1, \ldots, t_k)$}. \hfill$\spadesuit$ 

\end{definition}  
$k\text{-}\mathsf{D}\text{-}\mathsf{CMP}(t_1, \ldots, t_k)$ is for whether 
$k$ nodes form a simple defence-complete extension, and $kN\text{-}\mathsf{W}\text{-}\mathsf{D}\text{-}\mathsf{CMP}(t_1, \ldots, t_k)$ for whether they form a wide defence-complete extension.   
$k\text{-}\mathsf{E}\text{-}\mathsf{CMP}(t_1, \ldots, t_k)$ is for whether they 
form a simple and wide equivalence-complete extension (Cf. Theorem \ref{thm_collapse}).  

\begin{proposition}[Characterisability of $\sigma\tau$-complete extensions]\label{prop_sigma_tau_complete_extensions} 
    Given {\small $\mathcal{M}_{\sim}\linebreak \equiv 
	((ObjStmts,ObjE), Obj\Pi)$},   
	let $N$ be $|ObjStmts|$ and $k$ be $0 \leq k \leq N$. 
	Let $\evalit$ be: $\evalit(p_{\mathsf{A}})$ is 
	\resizebox{3cm}{!}{$[\star 1:\{\} \xrightarrow{\{attacks\}} \star 2:\{\}]$} if $p_{\mathsf{A}}$'s arity is 2 and   
\resizebox{3cm}{!}{$
[\!
\begin{tikzpicture}[baseline=(n.base)]
  \node (n) {$\star 1:\{\}$};
  \draw[->, looseness=8, out=10, in=-10]
    (n) to node[right, xshift=2pt] {\scriptsize$\{attacks\}$} (n);
\end{tikzpicture}
\!]
$}
	if it is 1; $\evalit(p_{\mathsf{D}})$ 
	is 
	\resizebox{3cm}{!}{$[\star 1:\{\} \ \star 2:\{\} \ \cdots \  \star n:\{\}] $} 
	with $n$ being the arity of $p_{\mathsf{D}}$; 
	and 
	$\evalit(p_{\mathsf{AnnoEq}})$ is 
	\resizebox{3cm}{!}{$[\star 1:\{\star 3\} \ \star 2:\{\star 3\}]$}.  
	Then $\{\evalit(c_1), \ldots,\evalit(c_k)\}$ is 
	\begin{itemize} 
		\item a simple defence-complete extension iff$^*$ $\mathcal{M}_{\sim}, \evalit \models  
			k\text{-}\mathsf{D}\text{-}\mathsf{CMP}(c_1, \ldots, c_k)$. 
       		\item a wide defence-complete extension iff$^*$ $\mathcal{M}_{\sim}, \evalit \models  
			kN\text{-}\mathsf{W}\text{-}\mathsf{D}\text{-}\mathsf{CMP}(c_1, \ldots, c_k)$. 
      		\item a simple equivalence-complete extension iff$^*$   
			$\mathcal{M}_{\sim}, \evalit \models k\text{-}\mathsf{E}\text{-}\mathsf{CMP}(c_1, \ldots, c_k)$. 
	\end{itemize} 
\end{proposition} 
\begin{proof} 
	Let $\overrightarrow{c}$ denote $c_1, \ldots, c_k$.  

	We discharge the first obligation.

	\textbf{Only if}: By assumption, $\{\evalit(c_1), \ldots, \evalit(c_k)\}$ is a simple-admissible set. 
	We denote this set by $ObjStmts'$. By Proposition \ref{prop_admissible}, $\mathcal{M}_{\sim}, \evalit \models k\text{-}\mathsf{ADM}(\overrightarrow{c})$. 
	 
	 We prove that $\mathcal{M}_{\sim}, \evalit \models \forall x.(k\text{-}\mathsf{DF}(x, \overrightarrow{c}) \supset 
	 \bigvee_{l \leq k} x = c_l)$.  
	 Let $\evalit'$ be almost exactly $\evalit$ except that $\evalit'(x)$ may be different from $\evalit(x)$.  

	 \begin{description} 
		 \item[\textbf{Case $\evalit'(x) \not\in ObjStmts$}:] $\mathcal{M}_{\sim}, \evalit' \not\models 
			 k\text{-}\mathsf{DF}(x, \overrightarrow{c})$. 
		 \item[\textbf{Case $\evalit'(x) \in ObjStmts$}:] By Lemma \ref{lem_k_df}, 
			 $ObjStmts'$ simple-defends $\evalit'(x)$ iff$^*$ 
			 $\mathcal{M}_{\sim}, \evalit' \models k\text{-}\mathsf{DF}(x, \overrightarrow{c})$. 
			 Suppose $ObjStmts'$ simple-defends $\evalit'(x)$. 
			 By assumption, $\evalit'(x)$ is a member of $ObjStmts'$. So, 
			 $\mathcal{M}_{\sim}, \evalit' \models k\text{-}\mathsf{DF}(x, \overrightarrow{c}) \supset 
	 \bigvee_{l \leq k} x = c_l$.  
			 Suppose $ObjStmts'$ does not simple-defend $\evalit'(x)$. 
			 Then, $\mathcal{M}_{\sim}, \evalit' \not\models k\text{-}\mathsf{DF}(x, \overrightarrow{c})$. 
	 \end{description} 
	 Consequently, $\mathcal{M}_{\sim}, \evalit \models \forall x.(k\text{-}\mathsf{DF}(x, \overrightarrow{c}) \supset 
	 \bigvee_{l \leq k} x = c_l)$, as required. \\

	 \textbf{If}: We have been rigorous about the treatment of `iff$^*$'. To avoid excessive repetition, we omit the part of the proof. 
	 By assumption, $\mathcal{M}_{\sim}, \evalit \models k\text{-}\mathsf{ADM}(\overrightarrow{c})$. So, by Proposition 
	 \ref{prop_admissible}, $\{\evalit(c_1), \ldots, \evalit(c_k)\}$ is a simple-admissible set. We denote this set by $ObjStmts'$. 
	 By assumption, $\mathcal{M}_{\sim}, \evalit \models \forall x.(k\text{-}\mathsf{DF}(x, \overrightarrow{c}) \supset 
	 \bigvee_{l \leq k} x = c_l)$. So, for any $u \in ObjStmts$, if $ObjStmts'$ simple-defends $u$, then 
	 $u$ is in $ObjStmts'$, as required. \\ 

	 We discharge the second obligation.  

	 \textbf{Only if}: By assumption, $\{\evalit(c_1), \ldots, \evalit(c_k)\}$ is a wide-admissible set. 
	We denote this set by $ObjStmts'$. By Proposition \ref{prop_admissible}, $\mathcal{M}_{\sim}, \evalit \models 
	kN\text{-}\mathsf{WADM}(\overrightarrow{c})$. 
	 
	 We prove that $\mathcal{M}_{\sim}, \evalit \models \forall x.(kN\text{-}\mathsf{WDF}(x, \overrightarrow{c}) \supset 
	 \bigvee_{l \leq k} x = c_l)$.  
	 Let $\evalit'$ be almost exactly $\evalit$ except that $\evalit'(x)$ may be different from $\evalit(x)$.  

	 \begin{description} 
		 \item[\textbf{Case $\evalit'(x) \not\in ObjStmts$}:] $\mathcal{M}_{\sim}, \evalit' \not\models 
			 kN\text{-}\mathsf{WDF}(x, \overrightarrow{c})$. 
		 \item[\textbf{Case $\evalit'(x) \in ObjStmts$}:] By Lemma \ref{lem_kn_wdf}, 
			 $ObjStmts'$ wide-defends $\evalit'(x)$ iff$^*$ 
			 $\mathcal{M}_{\sim}, \evalit' \models kN\text{-}\mathsf{WDF}(x,\linebreak \overrightarrow{c})$. 
			 Suppose $ObjStmts'$ wide-defends $\evalit'(x)$. 
			 By assumption, $\evalit'(x)$ is a member of $ObjStmts'$. So, 
			 $\mathcal{M}_{\sim}, \evalit' \models kN\text{-}\mathsf{WDF}(x, \overrightarrow{c}) \supset 
	 \bigvee_{l \leq k} x = c_l$.  
			 Suppose $ObjStmts'$ does not wide-defend $\evalit'(x)$. 
			 Then, $\mathcal{M}_{\sim}, \evalit' \not\models kN\text{-}\mathsf{WDF}(x, \overrightarrow{c})$. 
	 \end{description} 
	 Consequently, $\mathcal{M}_{\sim}, \evalit \models \forall x.(kN\text{-}\mathsf{WDF}(x, \overrightarrow{c}) \supset 
	 \bigvee_{l \leq k} x = c_l)$, as required. \\

	 \textbf{If}: 
	 By assumption, $\mathcal{M}_{\sim}, \evalit \models kN\text{-}\mathsf{WADM}(\overrightarrow{c})$. So, by Proposition 
	 \ref{prop_admissible}, $\{\evalit(c_1),\linebreak \ldots, \evalit(c_k)\}$ is a wide-admissible set. 
	 We denote this set by $ObjStmts'$. 
	 By assumption, $\mathcal{M}_{\sim}, \evalit \models \forall x.(kN\text{-}\mathsf{WDF}(x, \overrightarrow{c}) \supset 
	 \bigvee_{l \leq k} x = c_l)$. So, for any $u \in ObjStmts$, if $ObjStmts'$ wide-defends $u$, then 
	 $u$ is in $ObjStmts'$, as required. \\

         We discharge the third obligation. 

	 \textbf{Only if}: By assumption, $\{\evalit(c_1), \ldots, \evalit(c_k)\}$ is a simple-admissible set.   
	 We denote this set by $ObjStmts'$. 
	 By Proposition \ref{prop_admissible}, $\mathcal{M}_{\sim}, \evalit \models k\text{-}\mathsf{ADM}(\overrightarrow{c})$.  
	 By assumption, $ObjStmts' = \text{cl}_{\sim}(ObjStmts')$. 
	 By Lemma \ref{lem_kl_cl}, 
	  $\mathcal{M}_{\sim}, \evalit \models kk\text{-}\mathsf{CL}(\overrightarrow{c})$, as required. \\

	  \textbf{If}: We omit the `iff$^*$' part of the proof. By assumption,  
	  $\mathcal{M}_{\sim}, \evalit \models k\text{-}\mathsf{ADM}(\overrightarrow{c})$. By Proposition 
	  \ref{prop_admissible}, $\{\evalit(c_1), \ldots, \evalit(c_k)\}$ is a simple-admissible set.   
	  We denote this set by $ObjStmts'$. 
	  By assumption, $\mathcal{M}_{\sim}, \evalit \models kk\text{-}\mathsf{CL}(\overrightarrow{c})$. 
	  By Lemma \ref{lem_kl_cl}, $ObjStmts' = \text{cl}_{\sim}(ObjStmts')$, as required.  
	  \hfill$\Box$ 
\end{proof}

$\mathcal{M}_{\sim}$'s extensions are 
first-order characterisable with these formulas. 
By Theorem \ref{thm_characterisation_wide_defence_grounded_extensions}, 
a simple equivalence-x extension also characterises a wide equivalence-x extension. 
By a simple observation, a simple/wide x extension 
is both a simple/wide defence-x extension and an equivalence-x extension, {\it i.e.}   
it is very easily derivable. We therefore do not explicitly include the characterisation 
of simple/wide x extensions.

\begin{theorem}[Characterisability of $\mathcal{M}_{\sim}$'s extensions]\label{prop_fol_characterisability_dung_acceptability_semantics}   
	Given {\small $\mathcal{M}_{\sim} \equiv \linebreak
	((ObjStmts,ObjE), Obj\Pi)$}, 
 	let $\evalit$ be such that: $\evalit(p_{\mathsf{A}})$ is \resizebox{3cm}{!}{$[\star 1:\{\} \xrightarrow{\{attacks\}} \star 2:\{\}]$} if $p_{\mathsf{A}}$'s arity is 2 and   
\resizebox{3cm}{!}{$
[\!
\begin{tikzpicture}[baseline=(n.base)]
  \node (n) {$\star 1:\{\}$};
  \draw[->, looseness=8, out=10, in=-10]
    (n) to node[right, xshift=2pt] {\scriptsize$\{attacks\}$} (n);
\end{tikzpicture}
\!]
$}
	if it is 1; $\evalit(p_{\mathsf{D}})$ 
	is 
	\resizebox{3cm}{!}{$[\star 1:\{\} \ \star 2:\{\} \ \cdots \  \star n:\{\}] $} 
	with $n$ being the arity of $p_{\mathsf{D}}$;  and 
	$\evalit(p_{\mathsf{AnnoEq}})$ is 
	\resizebox{3cm}{!}{$[\star 1:\{\star 3\} \ \star 2:\{\star 3\}]$}.  
	Let $N$ be $|ObjStmts|$ and $k$ be $0 \leq k \leq N$. 
	Then $\{\evalit(c_1), \ldots,\evalit(c_k)\}$ is 
	\begin{enumerate}  
			{\small 
		\item a simple defence-complete extension iff$^*$ 
			$\mathcal{M}_{\sim}, \evalit \models k\text{-}\mathsf{D}\text{-}\mathsf{CMP}(c_1, \ldots, c_k)$. 
		\item a wide defence-complete extension iff$^*$  
			$\mathcal{M}_{\sim}, \evalit \models kN\text{-}\mathsf{W}\text{-}\mathsf{D}
			\text{-}\mathsf{CMP}(c_1, 
			\ldots, c_k)$.  
		\item a simple equivalence-complete extension iff$^*$
			$\mathcal{M}_{\sim}, \evalit \models k\text{-}\mathsf{E}\text{-}\mathsf{CMP}(c_1, \ldots, c_k)$.  \\
		\item a simple defence-preferred extension iff$^*$
			$\mathcal{M}_{\sim}, \evalit \models k\text{-}\mathsf{D}\text{-}\mathsf{CMP}(c_1,\ldots,c_k) \\
			\wedge \bigwedge_{k+1 \leq m \leq N}\neg 
			(\exists x_m.\ldots.\exists x_{N}.(k+1+N-m)\text{-}\mathsf{D}\text{-}\mathsf{CMP}(c_1,\ldots,c_k, 
			x_{m},
			\ldots,x_{N}))$. 
		\item  a wide defence-preferred extension iff$^*$
			$\mathcal{M}_{\sim}, \evalit \models kN\text{-}\mathsf{W}\text{-}\mathsf{D}
			\text{-}\mathsf{CMP}(c_1,\ldots,c_k) \\
			\wedge \bigwedge_{k+1 \leq m \leq N}\neg 
			(\exists x_m.\ldots.\exists x_{N}.(k+1+N-m)N\text{-}\mathsf{W}\text{-}
			\mathsf{D}\text{-}\mathsf{CMP}(c_1,\ldots,c_k, 
			x_{m},
			\ldots,x_{N}))$.  
		\item a simple equivalence-preferred extension iff$^*$
			$\mathcal{M}_{\sim}, \evalit \models k\text{-}\mathsf{E}\text{-}\mathsf{CMP}(c_1,\ldots,c_k) \\
			\wedge \bigwedge_{k+1 \leq m \leq N}\neg 
			(\exists x_m.\ldots.\exists x_{N}.(k+1+N-m)\text{-}\mathsf{E}\text{-}\mathsf{CMP}(c_1,\ldots,c_k, 
			x_{m},
			\ldots,x_{N}))$. 
			\\
		\item a simple defence-grounded extension iff$^*$
			$\mathcal{M}_{\sim}, \evalit \models 
			k\text{-}\mathsf{D}\text{-}\mathsf{CMP}(c_1,\ldots,c_k) \\ 
			\wedge 
			\bigwedge_{m \leq k-1}
			\neg (\exists x_1.\ldots.\exists x_{m}.((\bigwedge_{m' \leq m}(\bigvee_{n \leq k} c_n = x_{m'})) \wedge m\text{-}\mathsf{D}\text{-}\mathsf{CMP}(x_1, 
			\ldots,x_{m})))$.
		\item a wide defence-grounded extension iff$^*$
			 $\mathcal{M}_{\sim}, \evalit \models 
			kN\text{-}\mathsf{W}\text{-}\mathsf{D}\text{-}\mathsf{CMP}(c_1,\ldots,c_k) \\ 
			\wedge 
			\bigwedge_{m \leq k-1}
			\neg (\exists x_1.\ldots.\exists x_{m}.((\bigwedge_{m' \leq m}(\bigvee_{n \leq k} c_n = x_{m'})) \wedge mN\text{-}\mathsf{W}\text{-}\mathsf{D}\text{-}\mathsf{CMP}(x_1, 
			\ldots,x_{m})))$. 
		\item a simple equivalence-grounded extension iff$^*$
			$\mathcal{M}_{\sim}, \evalit \models 
			k\text{-}\mathsf{E}\text{-}\mathsf{CMP}(c_1,\ldots,c_k) \\ 
			\wedge 
			\bigwedge_{m \leq k-1}
			\neg (\exists x_1.\ldots.\exists x_{m}.((\bigwedge_{m' \leq m}(\bigvee_{n \leq k} c_n = x_{m'})) \wedge m\text{-}\mathsf{E}\text{-}\mathsf{CMP}(x_1, 
			\ldots,x_{m})))$. 
			\\
		\item a simple defence-stable extension iff$^*$
			$\mathcal{M}_{\sim}, \evalit \models  k\text{-}\mathsf{D}\text{-}\mathsf{CMP}(c_1,\ldots,c_k) \\
			\wedge \forall z.(p_{\mathsf{D}}(z) \wedge (\bigwedge_{1 \leq j\leq k} 
	      \neg z = t_j) \supset \exists x.((\bigvee_{l\leq k}
		      x = t_l) \wedge p_{\mathsf{A}}(x, z)))$. 
		\item a wide defence-stable extension iff$^*$
			$\mathcal{M}_{\sim}, \evalit \models kN\text{-}\mathsf{W}\text{-}\mathsf{D}\text{-}\mathsf{CMP}(c_1, \ldots, c_k) \\
			\wedge \forall z.(p_{\mathsf{D}}(z) \wedge (\bigwedge_{1 \leq j\leq k} 
	      \neg z = t_j) \supset \exists x.((\bigvee_{l\leq k}
		      x = t_l) \wedge p_{\mathsf{A}}(x, z)))$. 

		\item a simple equivalence-stable extension iff$^*$
			$\mathcal{M}_{\sim}, \evalit \models 
			k\text{-}\mathsf{E}\text{-}\mathsf{CMP}(c_1, \ldots, c_k)\\
			\wedge \forall z.(p_{\mathsf{D}}(z) \wedge (\bigwedge_{1 \leq j\leq k} 
	      \neg z = t_j) \supset \exists x.((\bigvee_{l\leq k}
		      x = t_l) \wedge p_{\mathsf{A}}(x, z)))$. 
			}
			\end{enumerate} 
\end{theorem}   
\begin{proof}         
	Let $\overrightarrow{c}$ denote
	$c_1, \ldots, c_k$.  

	Proposition \ref{prop_sigma_tau_complete_extensions} discharge the obligations 1 through 3. \\

	To discharge the obligations 4, for \textbf{Only if} direction, 
	we must prove that the absence of a strictly larger simple defence-complete set than 
	$\{\evalit(c_1), \ldots, \evalit(c_k)\}$ is characterised by 
	 $\bigwedge_{k+1 \leq m \leq N}\neg 
			(\exists x_m.\ldots.\exists x_{N}.(k+1+N-m)\text{-}\mathsf{D}\text{-}\mathsf{CMP}(c_1,\ldots,c_k,\linebreak
			x_{m},
			\ldots,x_{N}))$. Let $N$ be $k$, then it is $\top$. Let $N$ be $k < N$. 
			We must prove that each number in $\{(k+1), \ldots, N\}$ is obtained 
			from $(k+1+N-m)$ where $m$ is such that $k+1 \leq m \leq N$. 
			When $m = N$, $(k+1)$ is obtained. When $m = (N-1)$, $(k+2)$ is obtained, 
	 		and so on. When $m = k + 1$, $N$ is obtained, as required.   

	   \textbf{If}: We omit the `iff$^*$' part. By assumption, 
	$\mathcal{M}_{\sim}, \evalit \models k\text{-}\mathsf{D}\text{-}\mathsf{CMP}(\overrightarrow{c})$. 
	   So, by Proposition \ref{prop_sigma_tau_complete_extensions}, 
	    $\{\evalit(c_1), \ldots, \evalit(c_k)\}$ is a simple defence-complete extension.  
	    We denote this set by $ObjStmts'$. 
	    If $k = N$, there cannot be any strictly larger simple defence-complete extension, 
	    so $ObjStmts'$ is a simple defence-preferred extension. If $k < N$, 
	    by assumption, $\mathcal{M}_{\sim}, \evalit \models \bigwedge_{k+1 \leq m \leq N}\neg 
			(\exists x_m.\ldots\linebreak.\exists x_{N}.(k+1+N-m)\text{-}\mathsf{D}\text{-}\mathsf{CMP}(c_1,\ldots,c_k,
			x_{m},
			\ldots,x_{N}))$. So, $ObjStmts'$ is a maximal simple defence-complete extension, {\it i.e.} 
			a simple defence-preferred extension, as required. 

	   Obligations 5 and 6 are discharged similarly. \\

	   To discharge the obligation 7, for \textbf{Only if} direction, we must prove that 
	   the absence of a strictly smaller simple defence-complete set than $\{\evalit(c_1), \ldots,\linebreak \evalit(c_k)\}$
	   is characterised by $\bigwedge_{m \leq k-1}
			\neg (\exists x_1.\ldots.\exists x_{m}.((\bigwedge_{m' \leq m}(\bigvee_{n \leq k} c_n = x_{m'})) \wedge m\text{-}\mathsf{D}\text{-}\mathsf{CMP}(x_1,
			\ldots,x_{m})))$. Let $k$ be 0, then it is $\top$. 
			Let $k$ be $0 < k$. Then, $\bigwedge_{m \leq k-1}$ ensures to cover 
			every number smaller than $k$ and greater than 0. For any such $m$ with $1 \leq m \leq k-1$,  
			$\bigwedge_{m' \leq m}$ ensures to cover every number greater than 0 up to $m$. 
			Hence, for any $\evalit'$ which is almost exactly $\evalit$ except that, 
			for any $1 \leq m' \leq m$, $\evalit'(x_{m'})$ may be different from $\evalit(x_{m'})$, 
			$\mathcal{M}_{\sim}, \evalit' \models  
			\bigwedge_{m' \leq m}(\bigvee_{n \leq k} c_n = x_{m'})$ holds just when 
			each of 
			$\evalit'(x_1), \ldots, \evalit'(x_{m})$ is a member of 
			$\{\evalit'(c_1), \ldots, \evalit'(c_k)\}$, as required.  

			\textbf{If}: We omit the `iff$^*$' part. By assumption, 
			$\mathcal{M}_{\sim}, \evalit \models k\text{-}\mathsf{D}\text{-}\mathsf{CMP}(\overrightarrow{c})$. 
	   So, by Proposition \ref{prop_sigma_tau_complete_extensions}, 
	    $\{\evalit(c_1), \ldots, \evalit(c_k)\}$ is a simple defence-complete extension.  
	    We denote this set by $ObjStmts'$. 
	    If $0 = k$, there cannot be any strictly smaller simple defence-complete extension, 
	    so $ObjStmts'$ is a simple defence-grounded extension. If $0 < k$, 
	    by assumption, $\mathcal{M}_{\sim}, \evalit \models 
	    \bigwedge_{m \leq k-1}
			\neg (\exists x_1.\ldots\linebreak.\exists x_{m}.((\bigwedge_{m' \leq m}(\bigvee_{n \leq k} c_n = x_{m'})) \wedge m\text{-}\mathsf{D}\text{-}\mathsf{CMP}(x_1, 
			\ldots,x_{m})))$. 
			So, $ObjStmts'$ is a minimal simple defence-complete extension, {\it i.e.} 
			a simple defence-grounded extension, as required. 

	   Obligations 8 and 9 are discharged similarly. \\

	We discharge the obligation 10.

	\textbf{Only if}: By assumption, $\{\evalit(c_1), \ldots, \evalit(c_k)\}$ is a simple defence-complete extension.  
	We denote this set by $ObjStmts'$. 
	By Proposition \ref{prop_sigma_tau_complete_extensions},
	$\mathcal{M}_{\sim}, \evalit \models k\text{-}\mathsf{D}\text{-}\mathsf{CMP}(\overrightarrow{c})$. 

	We prove that $\mathcal{M}_{\sim}, \evalit \models  
	\forall z.(p_{\mathsf{D}}(z) \wedge (\bigwedge_{1 \leq j\leq k} 
	      \neg z = c_j) \supset \exists x.((\bigvee_{l\leq k}
		      x = c_l) \wedge p_{\mathsf{A}}(x, z)))$.  
	Let $\evalit'$ be  almost exactly $\evalit$ except that $\evalit'(z)$ may not be $\evalit(z)$.   
	 \begin{description} 
		 \item[\textbf{Case $\evalit'(z) \not\in ObjStmts$}:] $\mathcal{M}_{\sim}, \evalit' \not\models  
			p_{\mathsf{D}}(z) \wedge (\bigwedge_{1 \leq j\leq k} 
	      \neg z = c_j)$.  
      \item[\textbf{Case $\evalit'(z) \in ObjStmts'$}:]  $\mathcal{M}_{\sim}, \evalit' \not\models  
			p_{\mathsf{D}}(z) \wedge (\bigwedge_{1 \leq j\leq k} 
	      \neg z = c_j)$.  
		 \item[\textbf{Case, otherwise}:]   
			 $\mathcal{M}_{\sim}, \evalit' \models p_{\mathsf{D}}(z) \wedge (\bigwedge_{1 \leq j\leq k} 
	      \neg z = c_j)$. We show 
			 $\mathcal{M}_{\sim}, \evalit' \models \exists x.((\bigvee_{l\leq k}
		      x = c_l) \wedge p_{\mathsf{A}}(x, z))$. 
			 By assumption, there is some $u \in ObjStmts'$ such that $(u, \evalit'(z)) \in ObjE$.  
			 Let $\evalit''$ be almost exactly $\evalit'$ except that $\evalit''(x)$ is the node $u$.  
			 Then, $\mathcal{M}_{\sim}, \evalit'' \models 
			 (\bigvee_{l\leq k} x = c_l) \wedge p_{\mathsf{A}}(x, z)$. 
	 \end{description} 
	 Consequently, $\mathcal{M}_{\sim}, \evalit \models 
	\forall z.(p_{\mathsf{D}}(z) \wedge (\bigwedge_{1 \leq j\leq k} 
	      \neg z = c_j) \supset \exists x.((\bigvee_{l\leq k}
		      x = c_l) \wedge p_{\mathsf{A}}(x, z)))$, as required.  

	 \textbf{If}: We omit the `iff$^*$' part. 
	 By assumption, $\mathcal{M}_{\sim}, \evalit \models k\text{-}\mathsf{D}\text{-}\mathsf{CMP}(\overrightarrow{c})$. 
	 So, by Proposition 
	 \ref{prop_sigma_tau_complete_extensions}, $\{\evalit(c_1), \ldots, \evalit(c_k)\}$ is a simple defence-complete 
	 extension. We denote this set by $ObjStmts'$. 
	 By assumption, $\mathcal{M}_{\sim}, \evalit \models 
	\forall z.(p_{\mathsf{D}}(z) \wedge (\bigwedge_{1 \leq j\leq k} 
	      \neg z = c_j) \supset \exists x.((\bigvee_{l\leq k}
		      x = c_l) \wedge p_{\mathsf{A}}(x, z)))$. 

		      So, for any $u \in ObjStmts \backslash ObjStmts'$, 
		      there is some $u' \in ObjStmts'$ such that $(u', u) \in ObjE$, as required. \\ 

	 Obligations 11 and 12 are discharged similarly. 
		 \hfill$\Box$ 
\end{proof} 
Note this characterisation involves no infinite components. 
This means every \textsf{FOL} formula in Theorem 
\ref{prop_fol_characterisability_dung_acceptability_semantics} 
has a corresponding propositional logic formula. 
\begin{theorem}[Propositional characterisability of $\mathcal{M}_{\sim}$'s extensions]{\ }\\ 
   $\mathcal{M}_{\sim}$'s extensions are characterisable in propositional logic.    
\end{theorem} 
\begin{proof} 
    The domain of discourse is finite, and there are no infinite 
	components in $\mathcal{M}_{\sim}$. The convertibility is then 
	immediate: we expand each quantifier 
	(with $\wedge$s for $\forall$ and with $\vee$s for $\exists$) until no variables 
	remain, and replace each distinct predicate (with now no variable) with a 
	distinct propositional variable.  
	\hfill$\Box$ 
\end{proof} 
$\mathcal{M}_{\mathsf{dung}}$'s extensions are subsumed in $\mathcal{M}_{\sim}$'s 
extensions. Hence, the propositional characterisability of 
preferred and grounded extensions---left open in \cite{Doutre14}---is immediately answered. 
\begin{corollary}[Propositional characterisability of $\mathcal{M}_{\mathsf{dung}}$'s extensions]{\ }\\ 
	$\mathcal{M}_{\mathsf{dung}}$'s extensions are characterisable in propositional logic.     
\end{corollary} 

In Dung's argumentation theory, an {\it acceptability semantics} of some type is defined as the set of 
all extensions of the type. For example, the {\it complete semantics} is the set of all complete extensions in $\mathcal{M}_{\textsf{dung}}$. We prove that acceptability semantics in $\mathcal{M}_{\sim}$ of type `wide defence-complete' is characterisable 
in our discussion graph semantics. Characterisability of acceptability semantics of any other types is proved 
identically. Then, propositional characterisability of $\mathcal{M}_{\textsf{dung}}$'s acceptability 
semantics is 
an immediate corollary.  

We make use of formulas $k_1k_2\text{-}\mathsf{DISTINCT}(t_1, \ldots, t_{k_1}, t_{k_1+1}, t_{k_1+k_2})$ 
and $k_1\cdots\linebreak k_{m}N\text{-}\mathsf{W}\text{-}\mathsf{D}\text{-}\mathsf{CMPS}(t_1, \ldots, t_{k_1}, 
t_{k_1+1}, \ldots, t_{k_1+k_2}, \ldots, t_{(\Sigma_{i < m} k_i) + 1}, \ldots, t_{(\Sigma_{i < m} k_i) + k_m})$. 

\begin{definition}[$k_1k_2\text{-}\mathsf{DISTINCT}$] \label{def_distinct}  \rm 
	$k_1k_2\text{-}\mathsf{DISTINCT}(t_1, \ldots, t_{k_1}, t_{k_1+1}, \ldots, t_{k_1+k_2})$ is: 
	
	\begin{itemize} 
		\item $\bot$ if $0 = k_1 = k_2$. 
		\item $p_{\mathsf{D}}(t_1, \ldots, t_{k_2})$ if $k_1 = 0 \not= k_2$. 
		\item $p_{\mathsf{D}}(t_1, \ldots, t_{k_1})$ if $k_2 = 0 \not= k_1$. 
		\item $p_{\mathsf{D}}(t_1, \ldots, t_{k_1}) \wedge 
			p_{\mathsf{D}}(t_{k_1+1}, \ldots, t_{k_1+k_2}) 
			\wedge \exists w_1.(p_{\mathsf{D}}(w_1) \wedge (((\bigvee_{1 \leq k'_1 \leq k_1} w_1 = t_{k'_1}) 
			\wedge \forall w_2.(p_{\mathsf{D}}(w_2) \wedge (\bigvee_{k_1+1 \leqq k'_2 \leq k_1 + k_2} w_2 = t_{k'_2})
			\supset w_1 \not= w_2)) \vee ((\bigvee_{k_1 + 1 \leq k'_2 \leq k_1 + k_2} w_1 = t_{k'_2}) 
			\wedge \forall w_2.(p_{\mathsf{D}}(w_2) \wedge (\bigvee_{1 \leqq k'_1 \leq k_1} w_2 = t_{k'_1})
			\supset w_1 \not= w_2))))$, otherwise. \hfill$\spadesuit$ 
	\end{itemize}
\end{definition} 
$k_1k_2\text{-}\mathsf{DISTINCT}$ is for whether $k_1$ nodes are distinct, $k_2$ nodes are distinct, 
and they are distinct. 

\begin{definition}[$k_1\cdots k_m N\text{-}\mathsf{W}\text{-}\mathsf{D}\text{-}\mathsf{CMPS}$] \label{def_cmps} 
      $k_1\cdots k_{m}N\text{-}\mathsf{W}\text{-}\mathsf{D}\text{-}\mathsf{CMPS}(t_1, \ldots, t_{k_1},\linebreak
t_{k_1+1}, \ldots, t_{k_1+k_2}, \ldots, t_{(\Sigma_{i < m} k_i) + 1}, \ldots, t_{(\Sigma_{i < m} k_i) + k_m})$ is:  
	\begin{itemize} 
		\item $(\bigwedge_{1 \leq j_1 < m} (\bigwedge_{j_1 < j_2 \leq m} k_{j_1}k_{j_2}\mathsf{DISTINCT}(t_{(\Sigma_{i < 
			j_1} k_i) + 1}, \ldots, t_{(\Sigma_{i < j_1} k_i) + k_{j_1}},\linebreak 
			t_{(\Sigma_{i < j_2} k_k) + 1}, \ldots, t_{(\Sigma_{i < j_2} k_i) + k_{j_2}})))  \\ 
			\wedge\\ 
			(\bigwedge_{1 \leq j \leq m} k_jN\text{-}\mathsf{W}\text{-}\mathsf{D}
			\text{-}\mathsf{CMP}(t_{(\Sigma_{i < j} k_i) + 1}, \ldots, t_{(\Sigma_{i < j} k_i) + k_j}))\\
			\wedge\\
			(\bigwedge_{0 \leq j \leq N} \forall v_1.\ldots \forall 
			v_j.(jN\text{-}\mathsf{W}\text{-}\mathsf{D}\text{-}\mathsf{CMP}(v_1, \ldots, v_j) \supset\\ \quad \bigvee_{1 \leq m' \leq m} \neg jk_{m'}\text{-}\mathsf{DISTINCT}(v_1, \ldots, 
			v_j, t_{(\Sigma_{i < m'} k_i ) + 1}, \ldots, t_{(\Sigma_{i < m'} k_i ) + k_{m'}})))$. 
	\end{itemize} 
\end{definition} 
$k_1\cdots k_m N\text{-}\mathsf{W}\text{-}\mathsf{D}\text{-}\mathsf{CMPS}$ is for whether 
$m$ distinct sets of nodes are all distinct to each other and comprise all wide defence-complete extensions.

\begin{theorem}[Characterisability of $\mathcal{M}_{\sim}$'s wide defence-complete acceptability semantics]\label{thm_msim_wide_defence_complete_acceptability_semantics} 
     Given {\small $\mathcal{M}_{\sim} \equiv 
	((ObjStmts,ObjE), Obj\Pi)$}, 
 	let $\evalit$ be such that: $\evalit(p_{\mathsf{A}})$ is \resizebox{3cm}{!}{$[\star 1:\{\} \xrightarrow{\{attacks\}} \star 2:\{\}]$} if $p_{\mathsf{A}}$'s arity is 2 and   
\resizebox{3cm}{!}{$
[\!
\begin{tikzpicture}[baseline=(n.base)]
  \node (n) {$\star 1:\{\}$};
  \draw[->, looseness=8, out=10, in=-10]
    (n) to node[right, xshift=2pt] {\scriptsize$\{attacks\}$} (n);
\end{tikzpicture}
\!]
$}
	if it is 1; $\evalit(p_{\mathsf{D}})$ 
	is 
	\resizebox{3cm}{!}{$[\star 1:\{\} \ \star 2:\{\} \ \cdots \  \star n:\{\}] $} 
	with $n$ being the arity of $p_{\mathsf{D}}$;  and 
	$\evalit(p_{\mathsf{AnnoEq}})$ is 
	\resizebox{3cm}{!}{$[\star 1:\{\star 3\} \ \star 2:\{\star 3\}]$}.  
	Let $N$ be $|ObjStmts|$ and, for each $1 \leq i \leq m$, $k_i$ be $0 \leq k_i \leq N$.

     Then, $\{\{\evalit(c_1), \ldots, \evalit(c_{k_1})\}, \ldots, 
	\{\evalit(c_{(\Sigma_{j < m} k_j) + 1}), \ldots, \evalit(c_{(\Sigma_{j < m} k_j) + k_m})\}\}$ 
	is the wide defence-complete acceptability semantics iff$^*$ 
	$\mathcal{M}_{\sim}, \evalit \models 
	k_1\cdots k_{m}N\text{-}\mathsf{W}\text{-}\mathsf{D}\text{-}\\\mathsf{CMPS}(c_1, \ldots, c_{k_1},
 \ldots, c_{(\Sigma_{i < m} k_i) + 1}, \ldots, c_{(\Sigma_{i < m} k_i) + k_m})$. 
\end{theorem} 
\begin{proof} 
   Straightforward. \hfill$\Box$ 
\end{proof} 

For any other type of acceptability semantics, it is just the matter of replacing the formula 
characterising a wide defence-complete extension to another formula characterising 
an extension of that type. 

\begin{theorem}[\textsf{FOL} characterisability of $\mathcal{M}_{\sim}$'s acceptability semantics] 
	$\mathcal{M}_{\sim}$'s acceptability semantics are characterisable in \textsf{FOL} with discussion graph semantics. 
\end{theorem} 
\begin{proof} 
   Straightforward. \hfill$\Box$ 
\end{proof} 

\begin{theorem}[Propositional characterisability of $\mathcal{M}_{\sim}$'s acceptability semantics] 
    $\mathcal{M}_{\sim}$'s acceptability semantics are characterisable in propositional logic. 
\end{theorem} 
\begin{proof} 
   Straightforward. \hfill$\Box$ 
\end{proof} 
\begin{corollary}[Propositional characterisability of $\mathcal{M}_{\textsf{dung}}$'s acceptability semantics] 
	$\mathcal{M}_{\textsf{dung}}$'s acceptability semantics are characterisable in propositional logic. 
\end{corollary} 

\section{Conclusions and Related Work}

This paper has formulated the \emph{discussion-graph semantics} of first-order logic (\textsf{FOL}), presented 
$M_{\sim}$ as an \emph{equivalence-equipped Dung model}, defined its extensions, and established their \emph{first-order characterisability}. As an immediate consequence, we have shown that all of Dung’s extensions are characterisable in propositional logic, thereby closing 
the question regarding the propositional characterisability of 
preferred and grounded extensions, left open in \cite{Doutre14}. Furthermore, we have shown 
that all of $\mathcal{M}_{\sim}$'s and $\mathcal{M}_{\textsf{dung}}$'s acceptability semantics 
are both first-order and propositional characterisable. Propositional logic is hence expressive enough 
to encode Dung's argumentation entirely. 

Our formulation of discussion-graph semantics is \emph{top-down}, encompassing any object-level annotated graph. Moreover, it is a \emph{predicate logic}, allowing 
reasoning over object-level annotated graphs with variables and quantifiers. 
Since numerous solvers and tools already exist for \textsf{FOL}, remaining within its formal realm provides substantial practical benefits from established theoretical and computational techniques.

As for future work, this framework provides a promising foundation for \emph{interdisciplinary research} bridging formal discussion and argumentation with \emph{program analysis and verification}. Having established a specification language for reasoning about discussion graphs in general, the next step is to explore the integration of existing formal methods from program analysis and verification into reasoning about discussions.

\subsection*{Related Work}

Several proposals have been made for reasoning about argumentation within formal 
logic \cite{Villata12,Doutre14,Wooldridge05}. However, these approaches are 
generally bottom-up, tailored to Dung’s 
model~\cite{Villata12,Doutre14}, and often lack variables and quantifiers \cite{Doutre14}. 
Some introduce new logical connectives for {\it attack} and related notions \cite{Villata12}, 
or encode meta-properties of object-level graphs as propositional variables \cite{Doutre14,Wooldridge05}. As such, they are not readily extensible to reasoning about other discussion models.  
The embedding of meta-properties complicates the correspondence between syntax and semantics, 
while the absence of quantification leads to lengthy formal descriptions.  
This paper addresses these issues. 

The idea of a \emph{typed discussion graph} as a generalisation of 
formal argumentation models appears in \cite{Arisaka22}, where it serves 
as a framework for formalising fallacies. 
Here, we instead use it as the domain of discourse in \textsf{FOL}.

The consideration of \emph{equivalence-equipped Dung models} has also appeared 
earlier, notably in \emph{Block Argumentation} \cite{ArisakaSantini19,Arisaka24}, 
which defines equality among arguments. 
However, as far as we are aware, the generalisation of Dung’s extensions through 
the distinctions between simple and wide conflict-freeness and defence has not 
been undertaken before.

\subsection*{Acknowledgements} 
This work was supported by JSPS KAKENHI Grant Numbers 21K12028 and  
25K15245. We thank anonymous reviewers for helpful comments. 

\bibliographystyle{abbrv}
\bibliography{references}

\end{document}